\documentclass[11pt]{article}
\usepackage[margin=1in]{geometry}

\PassOptionsToPackage{numbers,sort&compress}{natbib}

\usepackage[utf8]{inputenc}   
\usepackage[T1]{fontenc}      
\usepackage{microtype}        

\usepackage{amsmath, amssymb, amsthm, amsfonts}
\usepackage{bbm}
\usepackage{nicefrac}         

\usepackage{booktabs}         
\usepackage{tabularx}
\usepackage{multirow}
\usepackage{multicol}

\usepackage{graphicx}
\usepackage{subcaption}
\usepackage{caption}
\usepackage{tikz}
\usepackage{wrapfig}
\usepackage{xcolor}

\usepackage{algorithm}
\usepackage{algorithmic}      

\usepackage{tcolorbox}
\tcbuselibrary{skins,breakable}
\usepackage{framed}
\usepackage{fancyvrb}

\usepackage[inline]{enumitem}

\usepackage{natbib}

\usepackage{hyperref}
\hypersetup{
    colorlinks=true,
    linkcolor=blue,
    citecolor=blue,
    urlcolor=blue
}
\usepackage{url}

\usepackage{etoc}
\usepackage{etoolbox}
\usepackage{titletoc}

\usepackage{authblk}

\makeatletter
  \newif\ifinappendix
  \inappendixfalse
  \let\origaddcontentsline\addcontentsline
  \renewcommand{\addcontentsline}[3]{%
    \ifinappendix
      \origaddcontentsline{#1}{#2}{#3}%
    \fi
  }
  \pretocmd{\appendix}{\inappendixtrue}{}{}
\makeatother

\newtheorem{theorem}{Theorem}[section]
\newtheorem{lemma}[theorem]{Lemma}

\theoremstyle{definition}

\theoremstyle{remark}

\tcbset{
  colback=gray!5,
  colframe=black!75!black,
  boxrule=0.5mm,
  arc=4mm
}

\newtcolorbox{promptbox}[1]{
    colback=gray!5,
    colframe=gray!70!black,
    coltitle=white,
    fonttitle=\bfseries\sffamily,
    rounded corners,
    arc=3pt,
    title=#1,
    left=10pt, right=10pt, top=10pt, bottom=10pt,
    fontupper=\ttfamily\small
}

\newcommand{\SN}[1]{\noindent{\textcolor{purple}{\textbf{SN:}}}}

\title{Disagreement-Guided Weak-to-Strong Routing}

\author[1]{Sima Noorani$^*$}
\author[1]{Shayan Kiyani$^*$}
\author[1]{George Pappas}
\author[1]{Hamed Hassani}
\affil[1]{University of Pennsylvania}

\date{}
\begin{document}

\maketitle

{\renewcommand{\thefootnote}{*}
\footnotetext{Equal contribution. Correspondence to: \texttt{\{nooranis, shayank\}@seas.upenn.edu}}}

\begin{abstract}
Small language models are increasingly deployed on edge devices and robotic systems to meet cost and latency constraints, yet their limited capabilities can make reliable deployment challenging. Stronger models can provide assistance when needed, but deciding when to route a weak-model response to a stronger model often relies on signals of response correctness. In open-ended and agentic tasks, however, obtaining such signals can require costly ground-truth supervision that is difficult to obtain at scale. We revisit the foundations of weak-to-strong routing and ask whether routing decisions can instead be guided by signals that do not require ground-truth supervision. We first characterize the optimal weak-to-strong routing policy, while allowing for multiple weak-model responses generated through local inference. Building on this characterization, we derive a principled approximation to the optimal routing based on model-to-model comparison. In particular, we introduce a notion of model disagreement that can be measured directly from model-generated responses, without ground-truth supervision.
Importantly, the approximation becomes increasingly accurate as the strong model becomes a better proxy for the ground-truth distribution. Across agentic and single-turn tasks, our approach effectively navigates the trade-off between routing cost and response accuracy without ground-truth supervision. Moreover, even when ground-truth correctness signals are available for training, disagreement-based routing can outperform methods trained directly to predict correctness.
\end{abstract}

\section{Introduction}
Small language models offer a practical path to deploying AI agents on edge devices, robots, and other systems constrained by cost, latency, connectivity, or privacy. Yet their limited capabilities make it difficult to rely on them across the full range of tasks they may encounter \citep{lu-etal-2025-demystifying,chen2024theoreticallimitationsmultilayertransformer}. This has motivated growing interest in systems that primarily rely on a weak LLM, but occasionally route difficult instances to a stronger model when additional capability is needed \citep{gupta2024languagemodelcascadestokenlevel,yue2024largelanguagemodelcascades,chen2023frugalgptuselargelanguage}.

The possibility of such routing changes how we should think about inference with weak LLMs. Rather than simply generating a response and returning it, inference now involves three decisions: whether to \emph{accept} the current response, \emph{spend additional local computation} on further reasoning or sampling, or \emph{seek assistance} from a stronger model at substantially higher cost. These choices introduce a new trade-off: \emph{Stopping local inference too early} miss improvements achievable through inexpensive additional computation; \emph{persisting locally for too long} waste time and resources when stronger assistance is ultimately needed; and \emph{routing too readily} undermines the cost and latency advantages of the weak model.

A natural way to navigate this trade-off is to assess response correctness through learned confidence scores, verifiers, or reward models \citep{lightman2023letsverifystepstep, wang2024mathshepherdverifyreinforcellms}. In practical agentic applications, however, correctness is often difficult to evaluate. For a robotic agent, assessing a proposed next action may require knowledge of the physical environment and consequences that become apparent after execution. For a computer-use agent, whether an action is appropriate may depend on the state of an external application, user preferences, or the downstream effects of modifying a file or sending a message. More broadly, in long-horizon agentic workflows, a locally plausible action may undermine the eventual outcome, violate a user constraint, or introduce an error that becomes apparent several steps later \citep{shinn2023reflexionlanguageagentsverbal, fan2026agentprocessbenchdiagnosingsteplevelprocess, lù2025agentrewardbenchevaluatingautomaticevaluations}. Obtaining correctness labels in such settings may require subjective or expert judgment, interaction with the environment, or evaluation of entire trajectories. Such supervision is costly and does not readily scale across tasks and deployment conditions. These challenges motivate our central question:

\begin{center}
\emph{Can we learn when to rely on weak-model inference and when to seek stronger assistance—without ground-truth correctness labels?}
\end{center}

To answer this question, we take a step back and shift the perspective from asking when the weak model needs help to asking about the \emph{value of using a stronger model}. From this perspective, routing is fundamentally a comparative decision: what matters is not simply whether the weak model's response is incorrect, but whether using the strong model instead is likely to improve the outcome. If we had access to the true likelihood that each response is correct, the value of using the strong model instead would naturally be captured by the difference
\begin{equation}
    \text{strong-model correctness likelihood}
    -
    \text{weak-response correctness likelihood}.
    \label{eq:informal-routing-gain}
\end{equation}
This difference represents the \emph{gain} in correctness likelihood from going from the current weak-model response to a strong-model response. In Section\ref{sec:routing_fundamentals}, we formalize a generalized version of ~\eqref{eq:informal-routing-gain} and characterize optimal routing within a broader framework that jointly accounts for weak-model test-time inference allocation.

Both correctness quantities in (\ref{eq:informal-routing-gain}), however, are difficult to access without ground-truth supervision. This motivates us to ask how (\ref{eq:informal-routing-gain}) can be approximated using only quantities available from the models themselves. A key observation is that (\ref{eq:informal-routing-gain}) depends on the \emph{difference} between the two correctness likelihoods, rather than requiring either likelihood individually. This naturally shifts our attention from model-to-ground-truth comparison to \emph{model-to-model comparison}, and in particular to how much the models \emph{disagree}, which can be measured directly from model-generated responses.

In Section~\ref{sec:approximation}, we derive a principled, ground-truth-free approximation to~(\ref{eq:informal-routing-gain}) using only model-generated responses. At a high level, the resulting score takes a simple form:
\begin{equation}
    \underbrace{\text{weak--strong disagreement}}_{\text{difference across models}}
    -
    \underbrace{\text{strong--strong disagreement}}_{\text{uncertainty within the strong model}}.
    \label{eq:informal-corrected-disagreement}
\end{equation}
The first term asks whether the strong model is likely to provide something different from the weak response. The second then accounts for how uncertain the strong model itself is, since a different response is less likely to represent a useful correction when the strong model is itself uncertain. This approximation is principled in a natural sense: it becomes increasingly accurate as the strong model's predictive distribution approaches the ground-truth distribution, and is exact when the two coincide. To this end, our contributions are summarized as follows:

Our contributions are summarized as follows:
\begin{itemize}[leftmargin=*, labelsep=0.5em]
    \item \textbf{Fundamentals of weak-to-strong routing.}
    In Section~\ref{sec:routing_fundamentals}, we characterize optimal routing within a broader framework that jointly accounts for weak-model test-time inference allocation, identifying the population quantities governing when to continue locally and when to seek stronger assistance.

    \item \textbf{Routing through model disagreement.}
    In Section~\ref{sec:approximation}, we derive a novel, ground-truth-free approximation to the optimal routing signal based on two disagreement signals: how much the weak and strong models disagree, and how much the strong model disagrees with itself. Importantly, (i) the approximation gap shrinks as the strong model becomes more accurate, and (ii) both signals can be \emph{efficiently learned from scalable model-generated data, without ground-truth supervision}.

    \item \textbf{Empirical evaluation in agentic and reasoning tasks.}
    In Section~\ref{sec:experiments}, we evaluate disagreement-based routing across \emph{multi-step agentic environments, where step-level correctness is difficult to define or supervise}, and controlled single-turn reasoning tasks \footnote{We release our code at \href{https://github.com/nooranisima/disagreement-weak-to-strong-routing}{https://github.com/nooranisima/disagreement-weak-to-strong-routing}}. Across both regimes, our learned disagreement-based signals consistently outperform label-free baselines, validating that disagreement-based approximation provides a reliable, scalable learning signal when ground-truth supervision is unavailable. 
\end{itemize}

\section{Related Work}
\label{sec:related_work}

We briefly discuss the works most closely related to our evaluation setting and baselines, and defer a discussion to Appendix~\ref{app:extended_related_works}.
\paragraph{Learning to defer and supervised routing.}
The fundamental problem of knowing when a model should abstain or seek help has been extensively studied under the umbrella of learning-to-defer and selective prediction \citep{mozannar2021consistentestimatorslearningdefer, jitkrittum2024doesconfidencebasedcascadedeferral}. With the advent of large language models, this challenge translates into determining when a weak model should autonomously escalate to a stronger one. In parallel to frameworks that study deferral for verification \citep{kiyani2026trustcheapcheckweak}, our focus is on routing model generations. Many current cascade methods rely on ground-truth correctness labels to train quality scorers, process verifiers, or token-level confidence predictors \citep{chen2023frugalgptuselargelanguage, gupta2024languagemodelcascadestokenlevel, aggarwal2025automixautomaticallymixinglanguage, chuang2025learningroutellmsconfidence, rabanser2025gatekeeperimprovingmodelcascades}. In our experiments, we represent this paradigm using a ground-truth supervised baseline highlighting that our proposed framework fundamentally eliminates the need for such expensive, task-specific supervision.

\textbf{Label-free routing signals}. To avoid human annotations, recent work extracts reliability signals directly from a single unlabeled generation using internal uncertainty metrics or hidden-state probes \citep{gupta2024languagemodelcascadestokenlevel, kossen2024semanticentropyprobesrobust, du2024haloscopeharnessingunlabeledllm}. A prevalent alternative assesses reliability by generating multiple responses to measure internal agreement or semantic overlap \citep{wang2023selfconsistencyimproveschainthought, yue2024largelanguagemodelcascades, kuhn2023semanticuncertaintylinguisticinvariances, Farquhar-Semantic-Entropy}. These two paradigms motivate our label-free baselines in both settings to demonstrate that our single-pass, disagreement-trained probe yields a better routing signal than standard internal uncertainty metrics.

\paragraph{Step-level routing in agentic tasks.} 
In sequential environments, determining when to escalate an action to a stronger model is typically guided by process-level verifiers and trajectory-level rewards \citep{kapoor2026trimhybridinferencetargeted, hemadri2026r2vagentteachingslms, sarukkai2026inferencetimedistillationcostefficientagents}, or frameworks for cost-aware decision support policies \citep{kiyani2026strategicdecisionsupportai}. Because these approaches rely on explicit step-level correctness labels or complex episode-success evaluations, they are difficult to scale in open-ended domains. We contrast these paradigms with our agentic experiments, demonstrating that our framework can route actions autonomously at each step by predicting offline action disagreement, entirely bypassing the need for process verifiers or trajectory rewards.

\section{Fundamentals of Weak-to-Strong Routing}
\label{sec:routing_fundamentals}


In this section, we study the problem of weak-to-strong routing jointly with the allocation of local inference computation.\footnote{For simplicity, we model local test-time computation through repeated weak-model sampling (e.g., best-of-$n$), although more general forms of test-time inference are possible.} We first study this problem at the population level, assuming access to the underlying probabilistic quantities, which allows us to characterize the optimal solution and identify the fundamental principles governing sampling, selection, and routing. In the next sections, we use these principles to develop practical methods that do not require access to such population~quantities.

\paragraph{Setup and sampling.}
Let $(P,Y,Y_w,Y_s)$ follow a joint distribution, where $P\in\mathcal P$
is a prompt, $Y\in\mathcal Y$ its correct answer, and
$Y_w,Y_s\in\mathcal Y$ are the weak- and strong-model responses,
respectively.\footnote{For simplicity, we treat $\mathcal Y$ as a discrete
response space, so equality is well defined. For LLM outputs, this can be
understood as equality up to semantic equivalence.}

A sampling policy $\sigma:\mathcal P\to\mathbb N_+$ specifies the number of
weak-model responses generated for each prompt. Given $P$, let
$Y_{w,1},\ldots,Y_{w,\sigma(P)}$ denote these responses, and define the
resulting history as
\[
H_\sigma=(P,Y_{w,1},\ldots,Y_{w,\sigma(P)}).
\]

\paragraph{Selection and routing.}
Given the sampled weak responses, we define two policies. The routing policy
decides whether to settle for one of the weak-model responses or query the
strong model, while the selection policy determines which weak-model response
to return if we settle for the weak model. Formally,
\[
\pi(H_\sigma)\in[0,1],
\qquad
\phi(H_\sigma)\in\{1,\ldots,\sigma(P)\},
\]
where $\pi(H_\sigma)$ is the probability of querying the strong model and
$\phi(H_\sigma)$ is the index of the selected weak response.

Together, $(\sigma,\phi,\pi)$ determine how much local inference to use,
which weak response to return, and when to seek strong-model assistance.
Our goal is to maximize the accuracy of the final response while controlling
the use of both weak- and strong-model inference. We impose separate budgets:
$B_w$ on the expected number of weak-model samples and $B_s$ on the expected
frequency of strong-model queries. This gives the following optimization
problem:
\begin{tcolorbox}[
    title={Budgeted Weak-to-Strong Routing},
    colback=blue!2!white,
    colframe=blue!45!black,
    colbacktitle=blue!8!white,
    coltitle=black,
    fonttitle=\bfseries,
    boxrule=0.6pt,
    arc=1.5mm,
    left=8pt,
    right=8pt,
    top=2pt,
    bottom=2pt
]
\begin{equation}
\label{eq:routing_problem}
\begin{aligned}
\max_{\sigma,\phi,\pi}\quad&
\mathbb{E}\!\left[
  (1-\pi(H_\sigma))
  \mathbf{1}\{Y_{w,\phi(H_\sigma)}=Y\}
  +\pi(H_\sigma)\mathbf{1}\{Y_s=Y\}
\right]\\[4pt]
\text{subject to}\quad&
\mathbb{E}[\sigma(P)]\leq B_w,
\qquad
\mathbb{E}[\pi(H_\sigma)]\leq B_s.
\end{aligned}
\end{equation}
\end{tcolorbox}

The expectations are over the prompt--answer distribution and model
randomness. The two constraints allow weak-model inference and strong-model
queries to be allocated across prompts according to their respective benefits.

\paragraph{Population quantities.}
To characterize the optimal policy, define
\begin{equation}
\label{eq:population_quantities}
q_p(y):=\Pr(Y=y\mid P=p),
\qquad
t_p:=\Pr(Y_s=Y\mid P=p).
\end{equation}
Here, $y$ is a fixed candidate response, $Y$ is the random ground-truth
answer, and $Y_s$ is a response drawn from the strong model. Thus,
$q_p(y)$ is the probability that $y$ is correct, while $t_p$ is the
probability that the strong-model response is correct.



\begin{theorem}
\label{thm:optimal_routing}
Fix any sampling policy $\sigma$ satisfying
$\mathbb{E}[\sigma(P)]\leq B_w$. For a realized history
$h_\sigma=(p,y_{w,1},\ldots,y_{w,\sigma(p)})$, define
\begin{equation}
\label{eq:routing_gain}
\Delta(h_\sigma)
:=
t_p-\max_{i\leq \sigma(p)} q_p(y_{w,i}).
\end{equation}
In words, $\Delta(h_\sigma)$ is the gain in correctness probability from
querying the strong model rather than returning the best available weak
response. The optimal selection and routing policies are
\begin{equation}
\label{eq:optimal_selection_routing}
\phi^*(h_\sigma)
\in
\operatorname*{arg\,max}_{i\leq \sigma(p)} q_p(y_{w,i}),
\qquad
\pi^*(h_\sigma)
=
\begin{cases}
1, & \Delta(h_\sigma)>\tau,\\
0, & \Delta(h_\sigma)<\tau,
\end{cases}
\end{equation}
where $\tau\geq0$ is a global threshold chosen to satisfy the strong-model
budget.\footnote{A more detailed version of the theorem in
Appendix~\ref{app:optimal_routing} also characterizes the optimal sampling
policy $\sigma$.}
\end{theorem}

\paragraph{Interpreting the optimal policies.}
From~\eqref{eq:optimal_selection_routing}, the optimal selection rule is
straightforward: among the sampled weak responses, it selects the one with the
highest probability of being correct. However, the optimal routing policy is more involved: it is a threshold rule over the score \(\Delta(h_\sigma)\) defined in~\eqref{eq:routing_gain}. This score measures the improvement in correctness probability from querying
the strong model, $t_p$, instead of returning the best weak response,
$\max_{i\leq \sigma(p)} q_p(y_{w,i})$. The policy routes whenever this improvement exceeds
the global threshold $\tau$, which controls the routing budget. Thus, routing
depends not only on how reliable the weak response is, but on how much the
strong model is expected to improve upon it.


\paragraph{One-shot sampling as a special case.}
When $\sigma(P)\equiv 1$, selection is trivial and routing is the only
remaining decision. In this case,~(\ref{eq:routing_gain}) reduces to
$
\Delta(h_\sigma)
=
t_p-q_p(y_{w,1})$.
This formalizes the accuracy-gain intuition introduced in
(\ref{eq:informal-routing-gain}), also recovering the known one-shot
characterization of
\citet{mozannar2021consistentestimatorslearningdefer,jitkrittum2024doesconfidencebasedcascadedeferral}.


\section{Approximations}
\label{sec:approximation}

The optimal routing rule derived in
Section~\ref{sec:routing_fundamentals}, particularly
equations~\ref{eq:routing_gain} and~\ref{eq:optimal_selection_routing},
depends on the conditional correctness probabilities $q_p$ and
$t_p$ defined in~\eqref{eq:population_quantities}. In many practical tasks, these population quantities are
unavailable without expensive human labels or exhaustive ground-truth
verifiers. In this section, we derive a principled approximation to the
optimal routing score in~\eqref{eq:routing_gain} that eliminates this
dependence on ground-truth correctness. We first present the resulting
routing score and its practical interpretation, and then derive it from the
population-optimal rule.

Our approximation, which we call \emph{corrected disagreement}, takes a
simple form. Recall that
$h_\sigma=(p,y_{w,1},\ldots,y_{w,\sigma(p)})$ contains the prompt and its
sampled weak-model responses, and that $Y_s$ and $Y_s'$ are independent
responses drawn from the strong model for the same prompt. We define the
corrected-disagreement routing score as
\begin{equation}
\label{eq:corrected-disagreement}
\boxed{
S_{\mathrm{CD}}(h_\sigma)
=
\underbrace{
\min_{i\leq\sigma(p)}
\Pr(Y_s\neq y_{w,i}\mid P=p)
}_{\text{weak--strong disagreement}}
-
\underbrace{
\Pr(Y_s\neq Y_s'\mid P=p)
}_{\text{strong-model self-disagreement}}
}.
\end{equation}

The first term measures the disagreement of the strong model with the
collection of sampled weak responses. It corresponds to selecting the weak
response that the strong model is least likely to disagree with. The second
term corrects for disagreement arising from uncertainty within the strong
model itself. Importantly, both terms can be learned from model-generated data without
ground-truth labels. We discuss how to learn these quantities in practice
in Section~\ref{sec:finite-sample-algorithms}.



We next show $S_{\mathrm{CD}}$ arises as a principled
approximation to the population-optimal routing rule.

\subsection{Deriving Corrected Disagreement}
\label{sec:corrected-disagreement}

The population-optimal routing gain
in~\eqref{eq:routing_gain} is
$
\Delta(h_\sigma)
=
t_p-\max_{i\leq\sigma(p)}q_p(y_{w,i}),
$
where $q_p$ and $t_p$ are defined in~\eqref{eq:population_quantities}.


Let
$
s_p(y):=\Pr(Y_s=y\mid P=p)
$
denote the strong model's predictive distribution. Then
\begin{align}
\Delta(h_\sigma)
&\overset{(a)}{=}
\sum_{y\in\mathcal Y}q_p(y)s_p(y)
-\max_{i\leq\sigma(p)}q_p(y_{w,i})
\overset{(b)}{\approx}
\sum_{y\in\mathcal Y}s_p(y)^2
-\max_{i\leq\sigma(p)}s_p(y_{w,i})
\nonumber\\
&\overset{(c)}{=}
\min_{i\leq\sigma(p)}
\Pr(Y_s\neq y_{w,i}\mid P=p)
-\Pr(Y_s\neq Y_s'\mid P=p)
=
S_{\mathrm{CD}}(h_\sigma).
\label{eq:corrected-disagreement-derivation}
\end{align}
Here, $(a)$ uses the conditional independence
$Y_s\perp\!\!\!\perp Y\mid P$ to write
$t_p=\sum_y q_p(y)s_p(y)$.
In $(b)$, we replace the unknown ground-truth distribution $q_p$ by the
strong model's predictive distribution $s_p$.
Finally, $(c)$ follows from
$\Pr(Y_s\neq y\mid P=p)=1-s_p(y)$ and
$\Pr(Y_s\neq Y_s'\mid P=p)
=1-\sum_y s_p(y)^2$.
Thus, corrected disagreement approximates the population-optimal routing
gain using only comparisons between model-generated responses.

\paragraph{Approximation quality.}
The only approximation in
\eqref{eq:corrected-disagreement-derivation} occurs in step~$(b)$, where the
unknown ground-truth distribution $q_p$ is replaced by the strong model's
predictive distribution $s_p$. As the strong model becomes a better proxy
for the ground-truth distribution, this approximation becomes more accurate.
In the ideal case $s_p=q_p$, step~$(b)$ is exact, and corrected disagreement
recovers the population-optimal routing gain:
$
S_{\mathrm{CD}}(h_\sigma)=\Delta(h_\sigma).
$
Thus, the approximation gap naturally shrinks as $s_p$ approaches $q_p$.

\subsection{Population Comparison of Routing Scores}
\label{subsec:population_comparison}

In this section, we compare the quality of different routing scores at the
population level. We use extensive Monte Carlo sampling to accurately estimate
the probabilities defining each score, allowing us to compare the underlying
routing signals independently of how they are learned in practice. For
simplicity, we consider the one-shot setting $\sigma(P)\equiv 1$. We compare
four routing signals:


\begin{figure}[ht!]
    \centering
\includegraphics[width=1\linewidth]{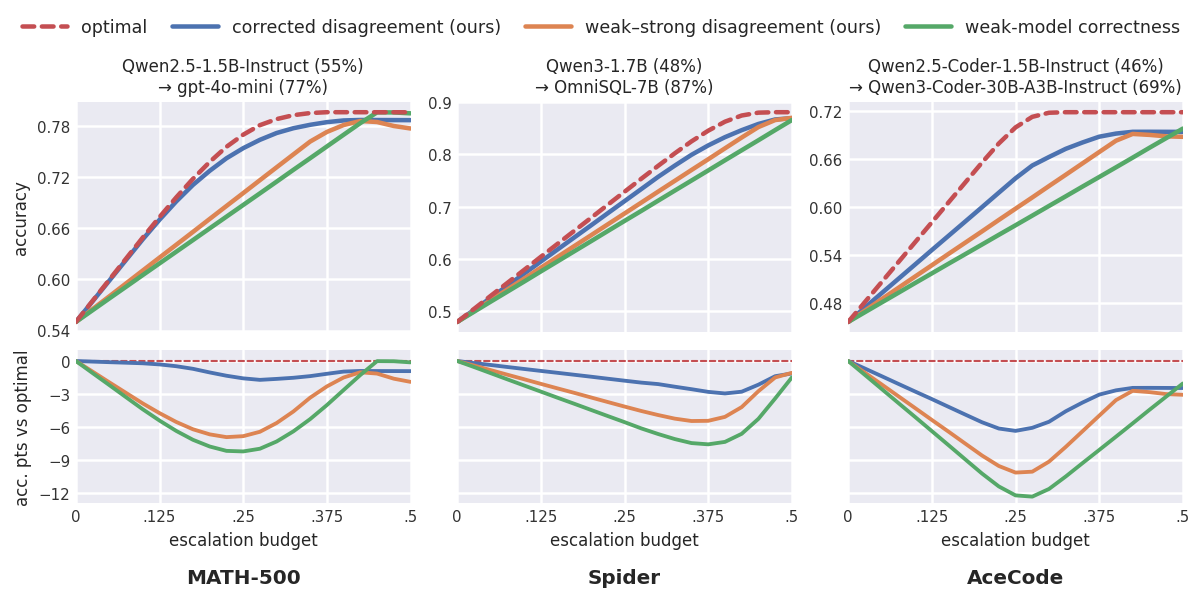}
    \caption{Corrected disagreement better approximates optimal routing than alternatives.}
    \label{fig:population-routing}
\end{figure}
\begin{itemize}[leftmargin=*, labelsep=0.5em]
    \item \textbf{Optimal routing gain:} $\Delta(p, y_w) = \Pr(Y_s = Y \mid P = p) - \Pr(Y = y_w \mid P = p)$. This is the exact gain defined in Theorem~\ref{thm:optimal_routing}. It serves as an oracle upper bound representing full knowledge of both models' correctness, and requires ground-truth supervision.
    \item \textbf{Weak-model correctness:} $S_{\mathrm{cor}}(p, y_w) = 1 - \Pr(Y = y_w \mid P = p)$. This represents optimal confidence-based routing, routing when the weak model is likely incorrect. It serves as the population oracle for confidence- and consistency-based methods and requires ground-truth labels.
    \item \textbf{Weak-strong disagreement:} $S_{\mathrm{disagree}}(p, y_w) = \Pr(Y_s \ne y_w \mid P = p)$. This baseline routes purely based on raw cross-model disagreement, without using ground-truth labels.
    \item \textbf{Corrected disagreement:} $S_{\mathrm{CD}}(p, y_w)$ (Equation~\ref{eq:corrected-disagreement}). This represents our proposed approximation and requires no ground-truth supervision.
\end{itemize}


Figure~\ref{fig:population-routing} reports results across the three varying benchmarks. In each panel, the top row displays weak-to-strong routing accuracy across informative escalation budget regime $B \in [0, 0.5]$, and the bottom row illustrates the gap relative to optimal routing. The results reveal three central insights:

\textbf{(i)} Corrected disagreement, $S_{\mathrm{CD}}$, consistently outperforms uncorrected weak--strong disagreement, $S_{\mathrm{disagree}}$, across all budgets. \textbf{(ii)} In operational regimes where the escalation budget is constrained ($B \le 0.4$), $S_{\mathrm{CD}}$ and $S_{\mathrm{disagree}}$ outperform the weak-model correctness baseline. The reason is, while correctness indicates when the weak model is wrong, it cannot anticipate whether the strong model will succeed. \textbf{(iii)} Because $S_{\mathrm{CD}}$ replaces ground truth with the strong model's predictive distribution, the approximation gap of $S_{\mathrm{CD}}$ is smaller for more accurate strong models.

\section{Finite-Sample Routing Algorithms}
\label{sec:finite-sample-algorithms}

We now turn the corrected-disagreement score into a practical routing
algorithm. We describe two components: first, how the learned score is used
at test time to route between the weak and strong models. Second, how
this score can be learned using only weak- and strong-model
generations.

\paragraph{Routing algorithm.}
Following Theorem~\ref{thm:optimal_routing}, we fix a sampling policy
$\sigma$ and a selection policy $\phi$. Recall from
Section~\ref{sec:routing_fundamentals} that $\sigma(p)$ specifies the number
of weak responses generated for prompt $p$, $\phi(h_\sigma)$ selects one of
these responses, and
$h_\sigma=(p,y_{w,1},\ldots,y_{w,\sigma(p)})$ denotes the resulting history.

Given a routing score $\widehat S(h_\sigma)\in\mathbb R$, test-time routing proceeds as follows:

\begin{tcolorbox}[
    colback=blue!2!white,
    colframe=blue!35!black,
    boxrule=0.5pt,
    arc=1.5mm,
    left=6pt,
    right=6pt,
    top=4pt,
    bottom=4pt
]
\begin{enumerate}
    \item Generate $\sigma(p)$ weak-model responses, forming $h_\sigma$.
    \item Compute the routing score $\widehat S(h_\sigma)$.
    \item If $\widehat S(h_\sigma)>\tau$, query the strong model; otherwise,
    return the weak response selected by $\phi(h_\sigma)$.
\end{enumerate}
\end{tcolorbox}
The threshold $\tau$ can be calibrated on held-out data to meet a desired strong-model budget $B_s$. More generally, the threshold can be updated online to maintain budget constraints over time, as in \citet{kiyani2026strategicdecisionsupportai}.

To recap, Theorem~\ref{thm:optimal_routing} shows that the optimal routing
rule is obtained by setting the routing score to the accuracy gain
$\Delta(h_\sigma)$ in~\eqref{eq:routing_gain}. In
Section~\ref{sec:corrected-disagreement}, we approximated this optimal score
by $S_{\mathrm{CD}}(h_\sigma)$ in~\eqref{eq:corrected-disagreement}, which,
unlike $\Delta(h_\sigma)$, depends only on model disagreement and does not
require ground-truth supervision. We now complete the framework by showing
how $S_{\mathrm{CD}}$ can be learned in practice using only model-generated
data.

\paragraph{Learning corrected disagreement.}
The corrected-disagreement score in
\eqref{eq:corrected-disagreement} consists of two quantities: weak--strong
disagreement and strong-model self-disagreement. We learn both offline using
only model-generated responses, without ground-truth labels.

For each training prompt $p$, we generate $K_w$ independent responses from
the weak model, $y_{w,1},\ldots,y_{w,K_w}$, and $K_s$ independent responses
from the strong model, $y_{s,1},\ldots,y_{s,K_s}$. Let
$A(y,y')\in\{0,1\}$ indicate whether two responses $y$ and $y'$ disagree. These generations allow us to construct empirical training targets for the
two terms appearing in $S_{\mathrm{CD}}$.

\emph{Weak--strong disagreement.}
For each weak response $y_{w,i}$, we define
\begin{equation}
\label{eq:disagreement-target}
\widehat d(p,y_{w,i})
:=
\frac{1}{K_s}\sum_{k=1}^{K_s} A(y_{w,i},y_{s,k}),
\qquad i=1,\ldots,K_w.
\end{equation}
This is the fraction of strong-model responses that disagree with
$y_{w,i}$, and hence provides an empirical estimate of the weak--strong
disagreement probability
$\Pr(Y_s\neq y_{w,i}\mid P=p)$.

\emph{Strong-model self-disagreement.}
We similarly define
\begin{equation}
\label{eq:self-disagreement-target}
\widehat g(p)
:=
\frac{1}{K_s(K_s-1)}
\sum_{k\neq k'} A(y_{s,k},y_{s,k'}).
\end{equation}
This is the fraction of pairs of strong responses that disagree, and hence
provides an empirical estimate of the strong-model self-disagreement
probability
$\Pr(Y_s\neq Y_s'\mid P=p)$.

The disagreement indicator $A$ depends on the task and only needs to
determine whether two responses are equivalent, rather than whether either
response is correct. For example, it can be implemented using exact matching,
task-specific parsing or execution, or an LLM judge. Importantly, $A$ is
used only offline to construct the training targets and is not required at
test time.

We train two predictors, $\widehat f_d(p,y_w)$ and $\widehat f_g(p)$, to
predict the targets in~\eqref{eq:disagreement-target} and
\eqref{eq:self-disagreement-target}, respectively. The first predicts the
probability that the strong model disagrees with a given weak response,
while the second predicts the strong model's self-disagreement from the
prompt alone. In Appendix~\ref{app:probes}, we provide implementation details for the predictors used in our experiments.




At test time, we combine the two predictors according to the
corrected-disagreement score:
\begin{equation}
\label{eq:learned-corrected-disagreement}
\widehat S_{\mathrm{CD}}(h_\sigma)
=
\min_{i\leq\sigma(p)}
\widehat f_d(p,y_{w,i})
-
\widehat f_g(p).
\end{equation}
Thus, $\widehat S_{\mathrm{CD}}$ can be computed using only the prompt and
the sampled weak-model responses. In particular, neither the strong model
nor ground-truth labels are required at test time.

\section{Experiments}
\label{sec:experiments}
In this section, we evaluate our disagreement-based routing framework across two settings. Section~\ref{subsec:agentic_experiments} first examines multi-step agentic environments, where step-level correctness is notoriously difficult to define. Section~\ref{subsec:finite_single_turn} considers single-turn reasoning tasks,
where gold answers provide a rigorous testbed to directly compare our approach
with ground-truth-supervised baselines.



\textbf{Baselines.} We benchmark against established routing signals, categorized by their test-time inference cost and reliance on ground-truth labels:

\textit{Single-response, label-free signals ($N=1$):} 
Evaluating a single weak response adds zero test-time generation overhead. We compare \textsc{Corrected disagreement (ours)} (Equation~\ref{eq:corrected-disagreement}); \textsc{Weak--strong disagreement (ours)}, isolating the value of the self-disagreement penalty; and \textsc{Token entropy} \citep{gupta2024languagemodelcascadestokenlevel}, averagin internal predictive uncertainty over the generated sequence.


\textit{Multi-response, label-free signals ($N>1$):} These assess confidence via weak-model self-agreement, a standard approach for cascade deferral \citep{yue2024largelanguagemodelcascades}. Though label-free, generating and clustering $N$ responses significantly increases inference overhead. We compare against \textsc{Self-consistency} \citep{wang2023selfconsistencyimproveschainthought}, routing via the plurality cluster's vote share, and \textsc{Semantic entropy} \citep{kuhn2023semanticuncertaintylinguisticinvariances}, routing via the cluster distribution's entropy.

\textit{Supervised reference:} For benchmark with ground-truth labels, we train a \textsc{Proxy reward} linear probe $\widehat{f}_c$. Requiring privileged supervision unavailable in open-ended tasks, it serves strictly as a supervised upper-bound reference rather than a fair label-free competitor.

\textbf{Evaluation metrics.} We evaluate cascade performance by plotting task success (final-answer accuracy for single-turn benchmarks, and end-to-end trajectory success for agentic environments) against the realized routing budget (the fraction of escalated queries or actions). To isolate the quality of the routing signal, we fix sampling and selection across all baselines and, at each routing budget, escalate the instances with the highest scores. This allows us to compare routing signals under the same sampling and selection policies without introducing error from threshold calibration. Selection ablations are deferred to Appendix~\ref{app:ablations-on-selection-policy}.

\subsection{Agentic Environments: Routing Without Ground Truth}
\label{subsec:agentic_experiments}
In multi-step environments, an agent completes a goal through a sequence of actions. Let $H_t = (p, y^1, o^1, \dots, y^{t-1}, o^{t-1})$ denote the trajectory history at step $t$, comprising the prompt $p$, past actions $y_i$, and observations $o_i$. Rather than escalating entire trajectories, we formulate routing at the \emph{action level}: given $H_t$, the router decides whether to execute the weak model's action locally ($y^t_w$) or defer that specific step to the strong model ($y^t_s$).
Defining step-level correctness in sequential environments is fundamentally difficult. Because tasks admit multiple valid paths, and expert evaluation is costly to scale\citep{lù2025agentrewardbenchevaluatingautomaticevaluations}. Furthermore, locally plausible actions may cause delayed downstream failures, making long-horizon credit assignment ambiguous \citep{fan2026agentprocessbenchdiagnosingsteplevelprocess}. Consequently, correctness-supervised routers and automated LLM judges remain notoriously unreliable in agentic domains \citep{shinn2023reflexionlanguageagentsverbal, zheng2023judgingllmasajudgemtbenchchatbot}.

While correctness is ambiguous, step-level \emph{disagreement} remains objective and easily verifiable. Let $\mathcal{T}(H_t, y)$ denote the transition triggered by action $y$ (e.g., a parsed tool call or state change). Two actions are equivalent if they induce identical transitions: $A(y_w, y_s) = \mathbbm{1}\{\mathcal{T}(H_t, y_w) = \mathcal{T}(H_t, y_s)\}$. Relying on observable mechanics rather than subjective evaluation enables effortless data collection from offline rollouts, entirely bypassing human labels and complex reward engineering.

\textbf{Datasets and models.} We evaluate two interactive environments, generating $K=10$ samples per step for training targets and multi-sample baselines. In \textsc{BFCL Multi-Turn} \citep{patil2025the} (160 train / 40 test tasks), Qwen3-8B routes to GPT-5-mini; action agreement requires identical function calls and arguments. In \textsc{WebShop} \citep{yao2023webshopscalablerealworldweb} (2,000 train / 500 test), Qwen3-8B routes to a task-specialized RL policy. To prevent distribution shifts by test-time strong-model interventions, we collect training data via an iterative, DAgger-style schedule \citep{li2026revisitingdaggererallmagents} (Appendix~\ref{app:dagger}).
\begin{figure}[ht!]
    \centering
    \includegraphics[width=1\linewidth]{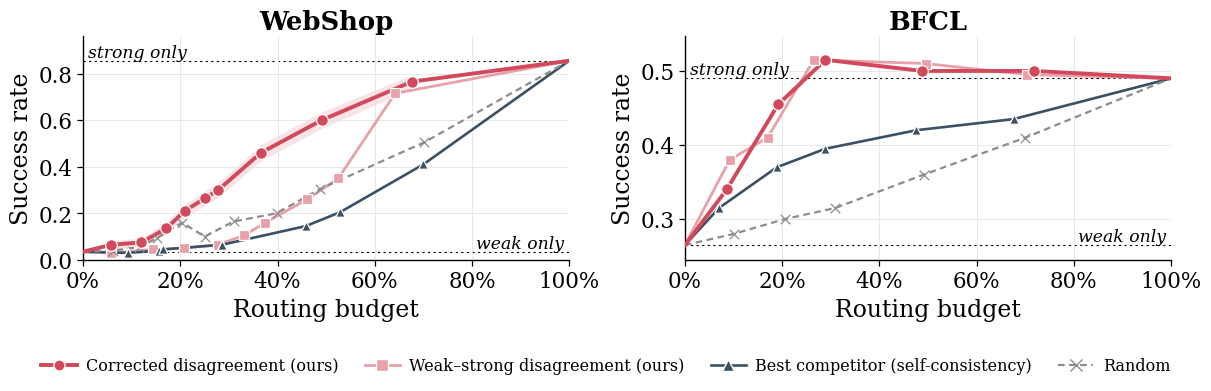}
    \caption{Action-level routing performance in agentic environments. For clarity, we plot our proposed signals against only the best-performing label-free baseline. Results are averaged over m=5 runs; comprehensive baselines, standard deviations, and tabular data are provided in Appendix~\ref{app:agentic_additional}.}.
    \label{fig:agentic-performance}
\end{figure}

Figure~\ref{fig:agentic-performance} shows both disagreement-based signals consistently outperform the best label-free baseline. Crucially, contrasting the two environments reveals how teacher capability dictates the value of the self-disagreement correction. In \textsc{WebShop}, where the strong model is highly capable, corrected disagreement significantly outperforms raw weak--strong disagreement. This reinforces our population-level finding: under a stronger teacher, the self-disagreement penalty prevents wasteful queries and the approximation approaches optimal routing. Conversely, in \textsc{BFCL}, the strong model itself struggles; here, the self-disagreement correction provides negligible benefit, rendering the gap between the two scores statistically insignificant. Beyond performance gains, the results highlight a practical advantage for agentic workflows. Because step-level correctness is inherently vague, determining exactly when an action requires escalation remains a persistent challenge. Relying instead on observable cross-model disagreement provides practitioners with a reliable, objective signal to safely guide routing decisions.

\subsection{Controlled Ablations on Single-Turn Benchmarks}
\label{subsec:finite_single_turn}

Single-turn reasoning benchmarks provide an ideal testbed for controlled ablations that are difficult to isolate in multi-step settings. Because these tasks feature canonical gold answers, they allow us to directly compare our label-free signals against a supervised proxy reward and study the impact of the weak-model sampling budget.


\textbf{Datasets and models.} We use three diverse datasets spanning a spectrum of evaluation complexity: \textsc{Math}-500 \citep{hendrycks2021measuringmathematicalproblemsolving} (Phi-3.5-mini-instruct $\to$ GPT-4o-mini; Disagreement indicator $A$ is based on exact final-answer match), \textsc{Spider} \citep{yu2019spiderlargescalehumanlabeleddataset} (Qwen3-1.7B $\to$ OmniSQL-7B; $A$ is identical SQL execution), and \textsc{AceCode} \citep{AceCoder} (Qwen2.5-Coder-1.5B-Instruct $\to$ Qwen3-Coder-30B-A3B-Instruct; $A$ is identical hidden unit test behavior). These tasks progress from closed-form verification to open-ended generation, where structurally distinct outputs can yield identical functional results. This spectrum allows us to observe how routing methods scale as as simple answer-matching becomes insufficient, reflecting the practical challenges of open-ended generations. Finally, to confirm that our findings hold across different capability gaps, we evaluate additional Weak--strong model pairings in Appendix~\ref{app:finite-sample-results}.

\begin{figure}[ht!]
    \centering
    \includegraphics[width=1\linewidth]{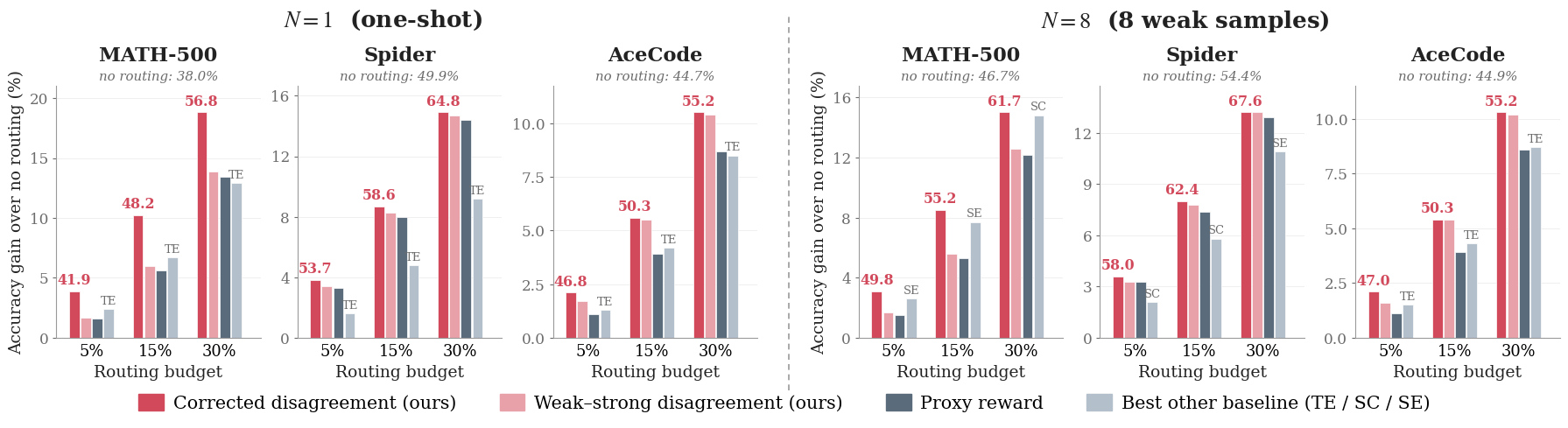}
    \caption{Finite-sample routing performance on single-turn benchmarks for one-shot ($N=1$) and multi-sample ($N=8$) regimes. Corrected disagreement consistently outperforms raw disagreement and the best label-free baselines. Multi-sample baselines only close the gap on closed-form tasks (MATH) where clustering is trivial, but break down on open-ended tasks (Spider, AceCode).}
    \label{fig:single-turn-performance}
\end{figure}
\textbf{Results and discussion.} 

Figure~\ref{fig:single-turn-performance} shows cascade accuracy at representative representative escalation budgets (further budgets in Appendix~\ref{app:finite-sample-results}). We compare the one-shot regime ($N=1$), where consistency scores are undefined, against a multi-sample regime using uniform sampling ($N=8$). For clarity, we plot our scores against the supervised proxy reward and only the strongest label-free baseline (exhaustive comparisons and additional budget values and sample sizes are deferred to Appendix~\ref{app:finite-sample-results}).

The empirical results validate our population-level analysis from Section~\ref{subsec:population_comparison}. First, at these practical budget constraints, corrected disagreement consistently outperforms the supervised proxy reward. Second, across all three benchmarks, corrected disagreement yields higher accuracy than raw weak--strong disagreement, confirming that penalizing strong-model self-inconsistency successfully prevents the cascade from wasting queries on ambiguous prompts.

Finally, the multi-sample ($N=8$) results expose the brittleness of consistency-based routing. On MATH-500, where answers are easily grouped, self-consistency leverages additional samples to narrow the gap with our approach. However, on \textsc{Spider} and \textsc{AceCode}, where semantic clustering is difficult, consistency baselines fail to scale. In these domains, our learned probe maintains superior performance without additional inference overhead of generating and clustering multiple responses.

\section{Conclusion}
In this work, we study weak-to-strong routing in its general setting, jointly accounting for test-time sampling of candidate weak responses, candidate selection, and escalation to a stronger model. Conventional cascades and standard routing approaches rely on signals of response correctness, which are often costly to obtain and ambiguous to define in open-ended or agentic tasks. To bypass the need for ground-truth supervision, we introduce a perspective based on model disagreement. We derive \emph{corrected disagreement}, a principled approximation to the optimal routing rule that balances cross-model disagreement against the strong model's internal uncertainty. This formulation provides an objective learning signal that can be collected at scale from model-generated responses alone. Across multi-step agentic environments and reasoning benchmarks, our learned disagreement signal consistently outperforms label-free baselines, showing that our principled approximation offers a practical and scalable path to guiding hybrid inference without ground-truth supervision.

\section{Acknowledgments}
This work was supported by EnCORE, the Institute for Emerging CORE Methods in Data Science, and NSF award 2502489 through the MFAI: Mathematical Foundations of Alignment in Generative Artificial Intelligence program. SK also received support through an AWS gift to Penn Engineering’s ASSET Center for Trustworthy AI.

\bibliography{iclr2027_conference}
\bibliographystyle{iclr2027_conference}

\appendix
\section{Appendix}

\subsection{Extended Related Works}
\label{app:extended_related_works}

As large language models are increasingly deployed, a broad literature has emerged around adaptive inference to balance response quality and computational cost. This design space generally splits into two paradigms: prompt-based routing, which selects a model before any text is generated \citep{ong2025routellmlearningroutellms, ding2024hybridllmcostefficientqualityaware}, and sequential cascades, which evaluate a cheaper model's output before deciding whether to escalate \citep{chen2023frugalgptuselargelanguage, yue2024largelanguagemodelcascades,jitkrittum2024doesconfidencebasedcascadedeferral}. While recent works such as \cite{dekoninck2025unifiedapproachroutingcascading} develop a unified formulation bridging both approaches, our work is specifically grounded in the cascade paradigm. Motivated by the growing need to deploy LLMs locally on resource-constrained edge devices \citep{liu2024mobilellmoptimizingsubbillionparameter,yu2024edgellmenablingefficientlarge,ravichandran2025distillingondevicelanguagemodels}, we study a two-model cascade with a significant capability gap. In this setting, a weak local model first generates candidate responses, and the system subsequently decides whether to invoke a more capable, expensive strong model based entirely on the realized generation history. To contextualize our contributions, the remainder of this section categorizes prior work along three critical axes: when the routing decision occurs, the required test-time inference compute, and the reliance on ground-truth supervision.


\textbf{Prompt-Based Routing and Learning to Defer}.
Prompt-based routers learn which model should answer an input without first observing its generated response. Building on classical literature regarding optimal deferral rules \citep{mozannar2021consistentestimatorslearningdefer, jitkrittum2024doesconfidencebasedcascadedeferral}, methods like Hybrid LLM \citep{ding2024hybridllmcostefficientqualityaware} and Route LLM \citep{ong2025routellmlearningroutellms} allocate queries between models using learned predictions of quality. Other approaches focus on selecting a suitable expert from a larger pool of LLMs by predicting model performance through meta-models \citep{_akota_2024}, distilling reward assessments to identify expertise \citep{lu2023routingexpertefficientrewardguided}, using input-dependent routing \citep{srivatsa2024harnessingpowermultipleminds}, estimating translation improvements using prompt representations \citep{luo2026routelmtlearnedsamplerouting}, or decomposing input uncertainty using labeled calibration data to avoid routing inherently ambiguous queries \citep{peale2026flexibleroutinguncertaintydecomposition}.
These methods fundamentally differ from our approach in two ways: they make routing decisions \emph{before} observing the generated weak response, and they heavily rely on ground-truth correctness or reward signals for training. In contrast, our cascade evaluates the realized weak response before routing, strictly in an unannotated two-model setting.

\textbf{Supervised Confidence and Deferral ($N=1$).}
Cascade methods commonly assess whether a single generated weak response is reliable enough to retain. FrugalGPT \citep{chen2023frugalgptuselargelanguage} trains a response-quality scorer to escalate when its confidence is insufficient. \cite{gupta2024languagemodelcascadestokenlevel} study token-level uncertainty and learned deferral predictors supervised directly by whether the strong model corrects a weak-model error. Automix \citep{aggarwal2025automixautomaticallymixinglanguage} utilizes a POMDP routing policy trained on self-verification judgments and ground-truth signals. Finally, SELF-REF \citep{chuang2025learningroutellmsconfidence} and Gatekeeper \citep{rabanser2025gatekeeperimprovingmodelcascades} train the weak model to express verbalized confidence or produce more useful confidence scores. All of these methods rely on task-specific correctness labels for supervision, and cannot be deployed in our purely label-free setting. Motivated by these works, our supervised weak-correctness probe \textsc{Proxy Reward} serves as a controlled oracle reference to represent this category of quality-based deferral. Our proposed framework assumes no access to ground-truth correctness or improvement signals.

\textbf{Consistency and Ensemble-Based Cascades ($N>1$ Label-Free).}
Consistency-based approaches assess reliability without labels by comparing multiple generated responses. \cite{yue2024largelanguagemodelcascades} use self-consistency(\citep{wang2023selfconsistencyimproveschainthought}) agreement among reasoning paths for deferral. Semantic Entropy \citep{kuhn2023semanticuncertaintylinguisticinvariances} and Kernel Language Entropy \citep{Farquhar-Semantic-Entropy} measure uncertainty over clusters of equivalent answers which can be used as deferral rules in cascades. Similarly, \cite{zhang2025leveraginguncertaintyestimationefficient} rely on clustering to compute the semantic entropies of both student and teacher models to construct routing preferences. Other methods evaluate consensus across an ensemble of multiple models: ABC \citep{kolawole2025agreementbasedcascadingefficientinference} uses pairwise similarities across inexpensive models, SMOOTHIE \citep{guha2024smoothielabelfreelanguage} estimates quality from distances between embedded outputs (which requires generator triplets), and \cite{niu2026routinggenerateddataannotationfree} groups queries into skill clusters via confidence-weighted ensemble consensus. These within-model consistency methods (particularly self-consistency and semantic entropy) directly motivate our multi-sample ($N>1$) baselines. However, as a broader category, these approaches fundamentally differ from our framework in their operational assumptions and mechanics. First, ensemble routers (\citep{guha2024smoothielabelfreelanguage,niu2026routinggenerateddataannotationfree}) are designed to select among a diverse pool of multiple models, whereas we study a strict cascade between a single local weak model and a more capable strong model. Consequently, our framework only requires the weak model's internal representations; the strong model can remain entirely black-box, accessed only for its text outputs during offline training. Second, consistency methods require generating multiple responses at inference time and rely heavily on semantic clustering. As we demonstrate in Section~\ref{subsec:finite_single_turn}, grouping responses into equivalence classes is notoriously brittle for open-ended tasks.

\textbf{Learned Label-Free Reliability and Cross-Model Verification.}
Recent work has explored learning reliability predictors from unlabeled outputs or cross-model comparisons. Semantic Entropy Probes \citep{kossen2024semanticentropyprobesrobust} predict semantic-entropy targets directly from model representations to enable single-response uncertainty estimation. HaloScope \citep{du2024haloscopeharnessingunlabeledllm} derives pseudo-labels from unlabeled response representations to train a truthfulness classifier. Additionally, \cite{xue2026verifyuncertainselfconsistencyblack}, \cite{hamidieh2026complementingselfconsistencycrossmodeldisagreement}, and \cite{gorbett2026crossmodeldisagreementlabelfreecorrectness} explore cross-model disagreement and verifier surprise to quantify uncertainty. These works provide strong conceptual motivation for comparing models and utilizing hidden-state probes. However, rather than detecting hallucinations or requiring test-time queries to a second model, our approach uniquely trains a probe on offline strong-model disagreement to directly approximate the optimal cascade routing gain, evaluating it autonomously at test time.

\textbf{Allocating Test-Time Computation.}
Additional inference computation can significantly improve response quality by generating and selecting among multiple candidates. For example, self-consistency \citep{wang2023selfconsistencyimproveschainthought} aggregates sampled reasoning paths, while broader test-time compute strategies \citep{snell2024scalingllmtesttimecompute} scale sampling and refinement based on input difficulty. For routing specifically, BEST-Route \citep{ding2025bestrouteadaptivellmrouting} pre-determines a sampling budget combined with model selection, and ModelSwitch \citep{chen2025trulyneedsamplesmultillm} evaluates consistency across multi-model samples to adaptively halt generation.  Our work studies when to escalate after observing the weak model's candidate responses. We characterize the expected improvement from replacing the selected weak response with a strong-model response and develop a disagreement-based approximation to guide this decision without ground-truth correctness labels.


\textbf{Sequential Tasks and Step-Level Routing.}
In sequential tasks, escalation decisions must be made repeatedly as the agent's state evolves. TRIM \citep{kapoor2026trimhybridinferencetargeted} uses process-level feedback to decide when to replace a weak reasoning step with a strong-model step. R2V-Agent \citep{hemadri2026r2vagentteachingslms} learns step-level escalation from weak-policy episode outcomes by combining uncertainty statistics with task-specific verifier features trained via verifier-guided DPO. Alternatively, Inference-Time Distillation \citep{sarukkai2026inferencetimedistillationcostefficientagents} determines escalation based on consistency among candidate actions to guide a weak agent. These methods typically rely on explicit step-level correctness labels, task-specific process verifiers, or expensive test-time consistency checks (which we benchmark via our multi-sample baselines). Our agentic experiments (Section~\ref{subsec:agentic_experiments}) show that by learning from objective action disagreement, we can outperform consistency based alternatives and effectively route at the step level without ever requiring episode-success labels or step-correctness annotations.

\subsubsection{Population estimates of routing signals across varying weak-strong model pairs}
We extend the population-level comparison in Section~\ref{subsec:population_comparison} to the full
routing-budget range $B\in[0,1]$ for completeness and additional weak--strong model
pairs and examine how the findings vary across
models, datasets, and numbers of weak-model samples. Figure~\ref{fig:app-population-n1} reports the one-shot setting, while Figures~\ref{fig:app-population-n4}
and~\ref{fig:app-population-n8} report results with $N=4$ and $N=8$
weak-model samples, respectively.

The results support the main observations: corrected disagreement
generally improves on raw weak--strong disagreement and closely
approximates optimal routing when the strong model is accurate.
Its advantage over weak-model correctness is most apparent at limited
routing budgets; the ordering can change at larger budgets.
We first report the one-shot setting ($N=1$), then vary the number
of weak-model samples $N$. While correctness flags when the weak model is wrong, it cannot anticipate whether the strong model will actually succeed. In practice, weak-model correctness is often viewed as a highly accessible baseline because the strong model can act as an automated judge to provide those labels. Conversely, obtaining ground-truth labels for the strong model's own correctness is notoriously difficult because it typically surpasses available automated judges. Our disagreement scores bypass this labeling bottleneck entirely while explicitly capturing potential accuracy gains. At very high escalation budgets, weak-model correctness slightly surpasses corrected disagreement because it strictly avoids routing those rare instances where the weak model is correct but the strong model fails. However, in typical weak-to-strong deployments, these failure modes are infrequent, making low-budget efficiency the governing practical metric.

\begin{figure}[t]
    \centering
    \includegraphics[width=\linewidth]{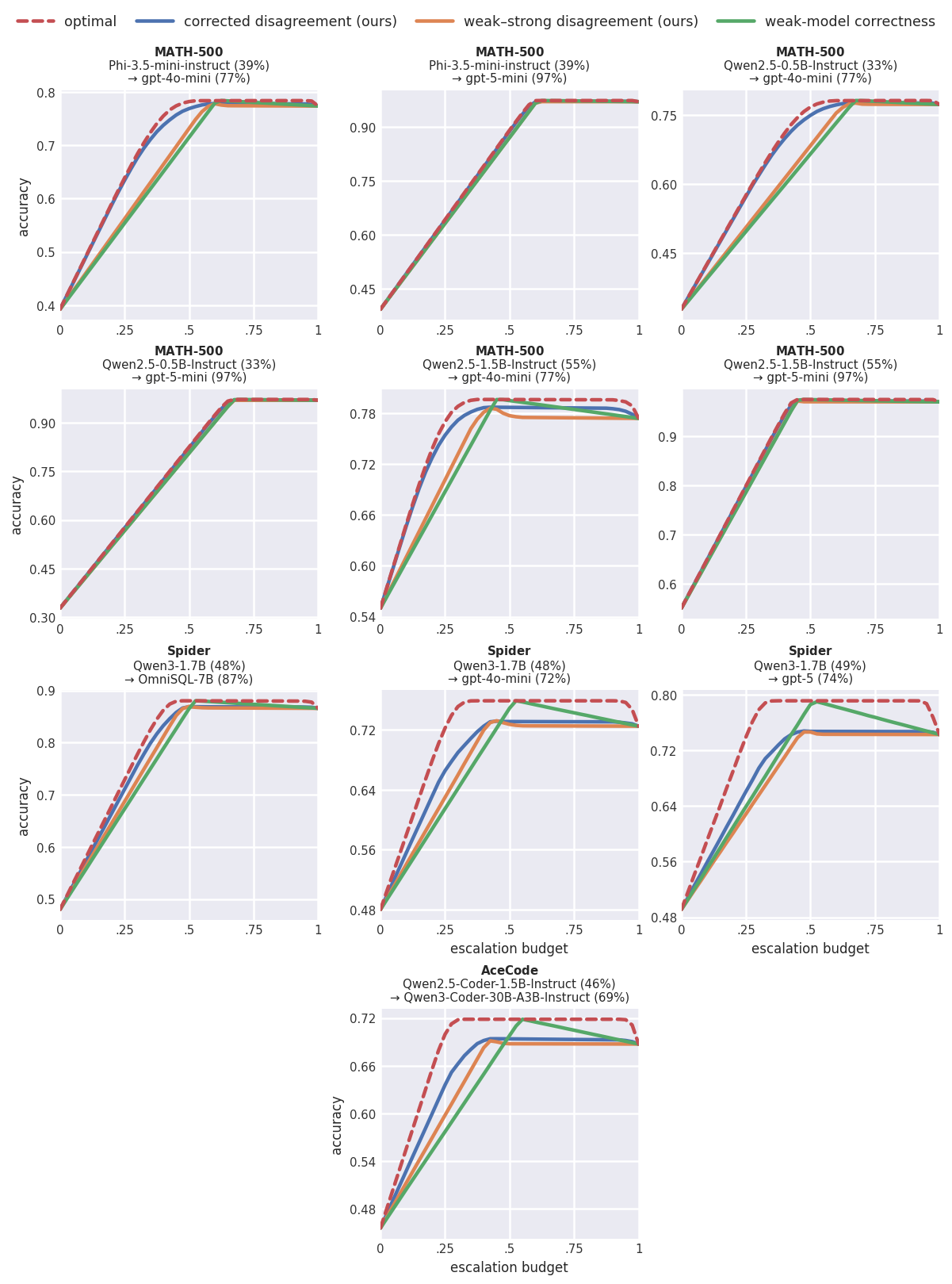}
    \caption{
    Population-level comparison of routing signals with $N=1$
    weak-model response. Each panel shows routing accuracy over
    the full budget range $B\in[0,1]$ for a different dataset and
    weak--strong model pair. We compare optimal routing,
    corrected disagreement, raw weak--strong disagreement,
    and weak-model correctness.
    }
    \label{fig:app-population-n1}
\end{figure}

\begin{figure}[t]
    \centering
    \includegraphics[width=\linewidth]{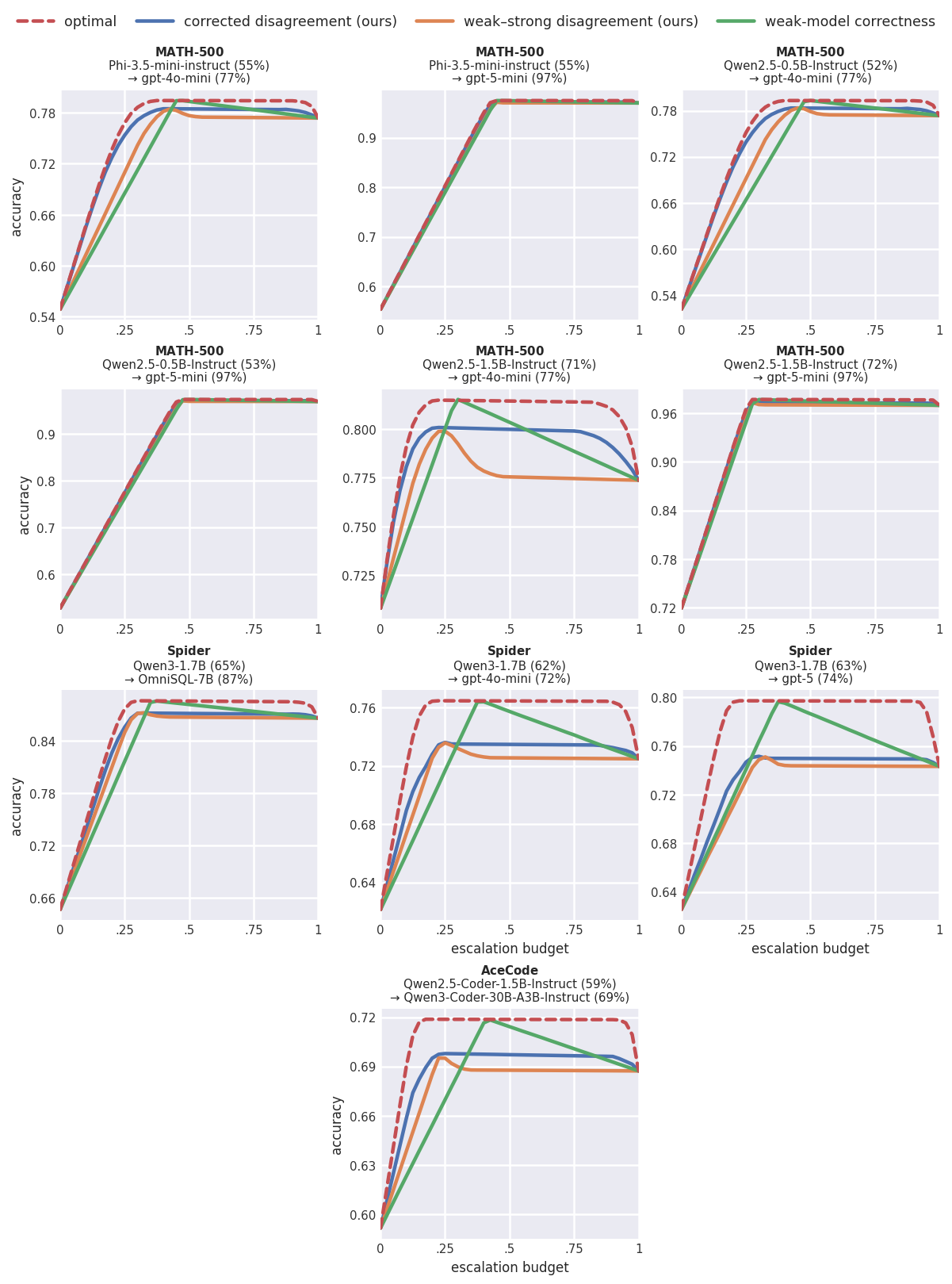}
    \caption{
    Population-level comparison of routing signals with $N=4$
    weak-model responses per prompt. Panels and routing signals
    follow Figure~\ref{fig:app-population-n1}.
    }
    \label{fig:app-population-n4}
\end{figure}

\begin{figure}[t]
    \centering
    \includegraphics[width=\linewidth]{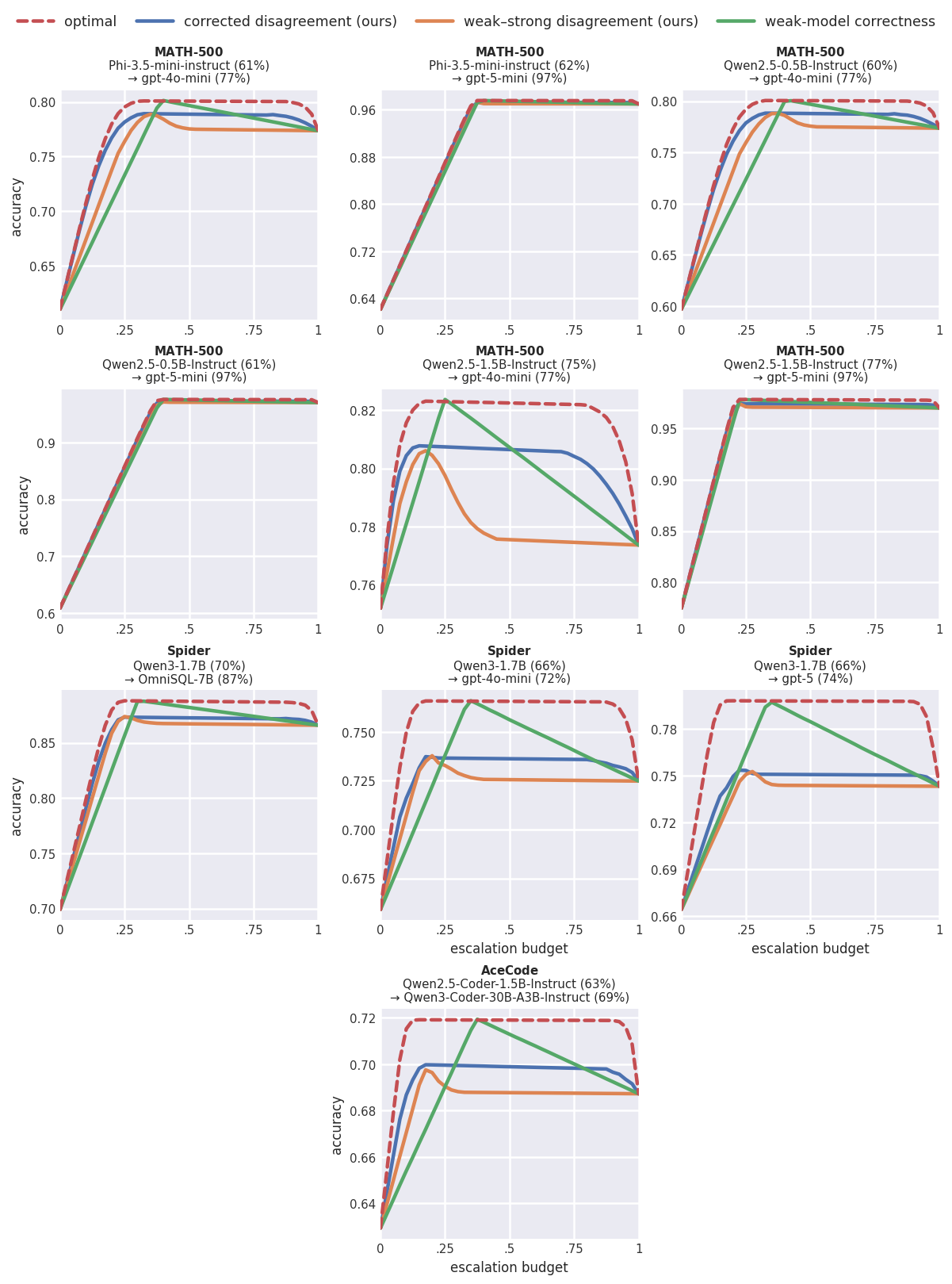}
    \caption{
    Population-level comparison of routing signals with $N=8$
    weak-model responses per prompt. Panels and routing signals
    follow Figure~\ref{fig:app-population-n1}.
    }
    \label{fig:app-population-n8}
\end{figure}


\subsection{Extended Finite-Sample Results Across Routing and Sampling Budgets}
\label{app:finite-sample-results}

We supplement Section~\ref{subsec:finite_single_turn} with detailed
results on \textsc{Math}-500, \textsc{Spider}, and \textsc{AceCode}
in Tables~\ref{tab:math500-qwen25-15b-gpt4omini},
\ref{tab:spider-qwen3-17b-omnisql7b},
and~\ref{tab:acecode-qwen25coder-15b-qwen3coder-10b}, respectively.
The tables include baselines omitted from the main-text plots, additional routing budgets, and $N\in\{1,4,8,16\}$ weak-model samples per prompt.

The results support the findings in the main text. On
\textsc{Math}-500, self-consistency and semantic entropy approach
corrected disagreement as the routing budget and sample count
increase, and slightly outperform it at the 50\% budget for $N>1$.
These consistency-based scores are less effective on \textsc{Spider}
and \textsc{AceCode}, where grouping equivalent responses is more
difficult. On \textsc{AceCode}, increasing $N$ provides little benefit
to either the selected weak response or consistency-based routing.
Disagreement-based scores remain competitive across all three
datasets, with corrected disagreement generally matching or
improving on raw weak--strong disagreement.


\begin{table*}[htp]
\centering
\small
\caption{\textsc{Math}-500 routing results for the weak--strong pair
\textbf{Phi-3.5-mini-instruct} $\rightarrow$ \textbf{GPT-4o-mini}.
The student answer is kept unless routed to the strong model. Each panel reports accuracy at
different routing budgets. The $b=0$ value is the accuracy of the kept weak answer before routing.}
\label{tab:math500-qwen25-15b-gpt4omini}
\begin{tabular}{llccccccc}
\toprule
Setting & Routing score
& $5\%$ & $10\%$ & $15\%$ & $20\%$ & $25\%$ & $30\%$ & $50\%$ \\
\midrule

\multicolumn{9}{l}{\textbf{$N=1$ weak sample per prompt} \quad ($b=0$: 0.380)} \\
Corrected disagreement & ours
    & \textbf{0.419} & \textbf{0.452} & \textbf{0.482} & \textbf{0.513} & \textbf{0.540} & \textbf{0.568} & \textbf{0.659} \\
Weak--strong disagreement & ours
    & 0.397 & 0.418 & 0.440 & 0.464 & 0.491 & 0.519 & 0.631 \\
Token entropy & baseline
    & 0.404 & 0.427 & 0.447 & 0.468 & 0.489 & 0.509 & 0.584 \\
Proxy reward$^\dagger$ & baseline
    & 0.396 & 0.415 & 0.436 & 0.460 & 0.487 & 0.514 & 0.629 \\
Oracle routing gain & oracle
    & 0.429 & 0.479 & 0.529 & 0.578 & 0.626 & 0.670 & 0.776 \\

\midrule
\multicolumn{9}{l}{\textbf{$N=4$ weak samples per prompt} \quad ($b=0$: 0.430)} \\
Corrected disagreement & ours
    & \textbf{0.465} & \textbf{0.495} & \textbf{0.524} & \textbf{0.551} & \textbf{0.576} & \textbf{0.599} & 0.677 \\
Weak--strong disagreement & ours
    & 0.447 & 0.467 & 0.489 & 0.512 & 0.539 & 0.563 & 0.664 \\
Self-consistency & baseline
    & 0.456 & 0.482 & 0.509 & 0.536 & 0.564 & 0.590 & \textbf{0.679} \\
Semantic entropy & baseline
    & 0.456 & 0.483 & 0.509 & 0.536 & 0.563 & 0.590 & 0.676 \\
Token entropy & baseline
    & 0.449 & 0.469 & 0.486 & 0.505 & 0.524 & 0.542 & 0.607 \\
Proxy reward$^\dagger$ & baseline
    & 0.445 & 0.463 & 0.484 & 0.507 & 0.533 & 0.558 & 0.661 \\
Oracle routing gain & oracle
    & 0.480 & 0.529 & 0.579 & 0.627 & 0.672 & 0.712 & 0.783 \\

\midrule
\multicolumn{9}{l}{\textbf{$N=8$ weak samples per prompt} \quad ($b=0$: 0.467)} \\
Corrected disagreement & ours
    & \textbf{0.498} & \textbf{0.524} & \textbf{0.552} & \textbf{0.575} & \textbf{0.596} & \textbf{0.617} & 0.686 \\
Weak--strong disagreement & ours
    & 0.484 & 0.503 & 0.523 & 0.547 & 0.571 & 0.593 & 0.680 \\
Self-consistency & baseline
    & 0.492 & 0.518 & 0.544 & 0.568 & 0.592 & 0.615 & \textbf{0.694} \\
Semantic entropy & baseline
    & 0.493 & 0.518 & 0.544 & 0.567 & 0.589 & 0.612 & 0.690 \\
Token entropy & baseline
    & 0.483 & 0.501 & 0.516 & 0.532 & 0.549 & 0.564 & 0.623 \\
Proxy reward$^\dagger$ & baseline
    & 0.482 & 0.499 & 0.520 & 0.543 & 0.566 & 0.589 & 0.680 \\
Oracle routing gain & oracle
    & 0.516 & 0.566 & 0.615 & 0.662 & 0.704 & 0.741 & 0.787 \\

\midrule
\multicolumn{9}{l}{\textbf{$N=16$ weak samples per prompt} \quad ($b=0$: 0.490)} \\
Corrected disagreement & ours
    & \textbf{0.518} & \textbf{0.541} & \textbf{0.567} & \textbf{0.588} & \textbf{0.609} & \textbf{0.628} & 0.691 \\
Weak--strong disagreement & ours
    & 0.507 & 0.525 & 0.544 & 0.566 & 0.588 & 0.609 & 0.687 \\
Self-consistency & baseline
    & 0.514 & 0.538 & 0.562 & 0.584 & 0.606 & 0.627 & \textbf{0.700} \\
Semantic entropy & baseline
    & 0.514 & 0.538 & 0.561 & 0.582 & 0.602 & 0.624 & 0.697 \\
Token entropy & baseline
    & 0.505 & 0.519 & 0.533 & 0.548 & 0.564 & 0.577 & 0.634 \\
Proxy reward$^\dagger$ & baseline
    & 0.505 & 0.522 & 0.542 & 0.562 & 0.585 & 0.606 & 0.687 \\
Oracle routing gain & oracle
    & 0.539 & 0.588 & 0.637 & 0.682 & 0.723 & 0.754 & 0.789 \\

\bottomrule
\end{tabular}
\vspace{0.5em}
\caption*{\footnotesize Consistency-based scores are undefined for $N=1$ and are therefore omitted in that panel.
Bold numbers mark the best non-oracle routing score at each budget within a fixed $N$.
$^\dagger$Proxy reward denotes the reward-model-based routing baseline.}
\end{table*}

\begin{table*}[htp]
\centering
\small
\caption{\textsc{Spider} routing results for the weak--strong pair
\textbf{Qwen3-1.7B} $\rightarrow$ \textbf{OmniSQL-7B}.
The weak query is kept unless routed to the strong model. Each panel reports accuracy at
different routing budgets. The $b=0$ value is the accuracy of the kept weak query before routing.}
\label{tab:spider-qwen3-17b-omnisql7b}
\begin{tabular}{llccccccc}
\toprule
Setting & Routing score
& $5\%$ & $10\%$ & $15\%$ & $20\%$ & $25\%$ & $30\%$ & $50\%$ \\
\midrule

\multicolumn{9}{l}{\textbf{$N=1$ weak sample per prompt} \quad ($b=0$: 0.499)} \\
Corrected disagreement & ours
    & \textbf{0.537} & \textbf{0.563} & \textbf{0.586} & \textbf{0.608} & \textbf{0.628} & \textbf{0.648} & \textbf{0.716} \\
Weak--strong disagreement & ours
    & 0.533 & 0.559 & 0.582 & 0.605 & 0.626 & 0.646 & \textbf{0.716} \\
Token entropy & baseline
    & 0.515 & 0.530 & 0.547 & 0.562 & 0.577 & 0.591 & 0.647 \\
Proxy reward$^\dagger$ & baseline
    & 0.532 & 0.557 & 0.579 & 0.601 & 0.623 & 0.643 & 0.714 \\
Oracle routing gain & oracle
    & 0.549 & 0.598 & 0.648 & 0.698 & 0.744 & 0.783 & 0.818 \\

\midrule
\multicolumn{9}{l}{\textbf{$N=4$ weak samples per prompt} \quad ($b=0$: 0.529)} \\
Corrected disagreement & ours
    & \textbf{0.565} & \textbf{0.591} & \textbf{0.612} & \textbf{0.632} & \textbf{0.650} & \textbf{0.668} & 0.728 \\
Weak--strong disagreement & ours
    & 0.561 & 0.586 & 0.608 & 0.629 & 0.647 & 0.666 & \textbf{0.729} \\
Self-consistency & baseline
    & 0.552 & 0.573 & 0.590 & 0.608 & 0.625 & 0.643 & 0.700 \\
Semantic entropy & baseline
    & 0.552 & 0.573 & 0.591 & 0.609 & 0.626 & 0.644 & 0.701 \\
Token entropy & baseline
    & 0.542 & 0.555 & 0.569 & 0.582 & 0.595 & 0.609 & 0.659 \\
Proxy reward$^\dagger$ & baseline
    & 0.561 & 0.585 & 0.605 & 0.625 & 0.645 & 0.663 & 0.727 \\
Oracle routing gain & oracle
    & 0.579 & 0.628 & 0.678 & 0.728 & 0.771 & 0.802 & 0.820 \\

\midrule
\multicolumn{9}{l}{\textbf{$N=8$ weak samples per prompt} \quad ($b=0$: 0.544)} \\
Corrected disagreement & ours
    & \textbf{0.580} & \textbf{0.604} & \textbf{0.624} & \textbf{0.642} & \textbf{0.660} & \textbf{0.676} & 0.734 \\
Weak--strong disagreement & ours
    & 0.577 & 0.601 & 0.622 & \textbf{0.642} & 0.659 & \textbf{0.676} & \textbf{0.735} \\
Self-consistency & baseline
    & 0.565 & 0.583 & 0.602 & 0.620 & 0.637 & 0.653 & 0.708 \\
Semantic entropy & baseline
    & 0.564 & 0.583 & 0.601 & 0.620 & 0.637 & 0.653 & 0.709 \\
Token entropy & baseline
    & 0.557 & 0.569 & 0.582 & 0.594 & 0.607 & 0.619 & 0.666 \\
Proxy reward$^\dagger$ & baseline
    & 0.577 & 0.600 & 0.618 & 0.636 & 0.655 & 0.673 & 0.732 \\
Oracle routing gain & oracle
    & 0.594 & 0.643 & 0.693 & 0.742 & 0.782 & 0.809 & 0.820 \\

\midrule
\multicolumn{9}{l}{\textbf{$N=16$ weak samples per prompt} \quad ($b=0$: 0.552)} \\
Corrected disagreement & ours
    & \textbf{0.588} & \textbf{0.612} & \textbf{0.631} & \textbf{0.649} & \textbf{0.666} & \textbf{0.682} & \textbf{0.737} \\
Weak--strong disagreement & ours
    & 0.585 & 0.608 & 0.629 & 0.647 & 0.664 & 0.681 & \textbf{0.737} \\
Self-consistency & baseline
    & 0.570 & 0.588 & 0.606 & 0.625 & 0.642 & 0.658 & 0.712 \\
Semantic entropy & baseline
    & 0.570 & 0.589 & 0.605 & 0.622 & 0.640 & 0.658 & 0.714 \\
Token entropy & baseline
    & 0.564 & 0.575 & 0.587 & 0.600 & 0.611 & 0.624 & 0.668 \\
Proxy reward$^\dagger$ & baseline
    & 0.584 & 0.607 & 0.625 & 0.643 & 0.661 & 0.677 & 0.734 \\
Oracle routing gain & oracle
    & 0.602 & 0.651 & 0.701 & 0.749 & 0.788 & 0.811 & 0.821 \\

\bottomrule
\end{tabular}
\vspace{0.5em}
\caption*{\footnotesize Consistency-based scores are undefined for $N=1$ and are therefore omitted in that panel.
Bold numbers mark the best non-oracle routing score at each budget within a fixed $N$.
$^\dagger$Proxy reward denotes the reward-model-based routing baseline.}
\end{table*}

\begin{table*}[htp]
\centering
\small
\caption{\textsc{AceCode} routing results for the weak--strong pair
\textbf{Qwen2.5-Coder-1.5B-Instruct} $\rightarrow$ \textbf{Qwen3-Coder-10B-A3B-Instruct}.
The weak program is kept unless routed to the strong model. Each panel reports accuracy at
different routing budgets. The $b=0$ value is the accuracy of the kept weak program before routing.}
\label{tab:acecode-qwen25coder-15b-qwen3coder-10b}
\begin{tabular}{llccccccc}
\toprule
Routing score & Type
& $5\%$ & $10\%$ & $15\%$ & $20\%$ & $25\%$ & $30\%$ & $50\%$ \\
\midrule

\multicolumn{9}{l}{\textbf{$N=1$ weak sample per prompt} \quad ($b=0$: 0.447)} \\
Corrected disagreement & ours
    & \textbf{0.468} & \textbf{0.486} & \textbf{0.503} & \textbf{0.519} & \textbf{0.535} & \textbf{0.552} & \textbf{0.608} \\
Weak--strong disagreement & ours
    & 0.464 & 0.484 & 0.502 & 0.518 & \textbf{0.535} & 0.551 & 0.606 \\
Token entropy & baseline
    & 0.460 & 0.475 & 0.489 & 0.504 & 0.518 & 0.532 & 0.591 \\
Proxy reward$^\dagger$ & baseline
    & 0.458 & 0.471 & 0.486 & 0.501 & 0.517 & 0.534 & 0.594 \\
Oracle routing gain & oracle
    & 0.497 & 0.547 & 0.596 & 0.646 & 0.690 & 0.708 & 0.709 \\

\midrule
\multicolumn{9}{l}{\textbf{$N=4$ weak samples per prompt} \quad ($b=0$: 0.449)} \\
Corrected disagreement & ours
    & \textbf{0.469} & \textbf{0.487} & \textbf{0.503} & \textbf{0.520} & 0.535 & \textbf{0.552} & \textbf{0.608} \\
Weak--strong disagreement & ours
    & 0.465 & 0.485 & \textbf{0.503} & 0.519 & \textbf{0.536} & \textbf{0.552} & 0.607 \\
Self-consistency & baseline
    & 0.460 & 0.472 & 0.484 & 0.495 & 0.507 & 0.519 & 0.565 \\
Semantic entropy & baseline
    & 0.459 & 0.470 & 0.482 & 0.494 & 0.506 & 0.518 & 0.564 \\
Token entropy & baseline
    & 0.461 & 0.476 & 0.490 & 0.505 & 0.520 & 0.535 & 0.592 \\
Proxy reward$^\dagger$ & baseline
    & 0.460 & 0.472 & 0.488 & 0.503 & 0.519 & 0.535 & 0.594 \\
Oracle routing gain & oracle
    & 0.500 & 0.550 & 0.599 & 0.649 & 0.692 & 0.709 & 0.710 \\

\midrule
\multicolumn{9}{l}{\textbf{$N=8$ weak samples per prompt} \quad ($b=0$: 0.449)} \\
Corrected disagreement & ours
    & \textbf{0.470} & \textbf{0.487} & \textbf{0.503} & \textbf{0.520} & \textbf{0.536} & \textbf{0.552} & \textbf{0.608} \\
Weak--strong disagreement & ours
    & 0.465 & 0.485 & \textbf{0.503} & 0.519 & 0.535 & 0.551 & 0.607 \\
Self-consistency & baseline
    & 0.459 & 0.471 & 0.482 & 0.494 & 0.505 & 0.517 & 0.563 \\
Semantic entropy & baseline
    & 0.459 & 0.470 & 0.482 & 0.493 & 0.506 & 0.517 & 0.564 \\
Token entropy & baseline
    & 0.464 & 0.478 & 0.492 & 0.507 & 0.521 & 0.536 & 0.592 \\
Proxy reward$^\dagger$ & baseline
    & 0.460 & 0.472 & 0.488 & 0.503 & 0.519 & 0.535 & 0.594 \\
Oracle routing gain & oracle
    & 0.499 & 0.549 & 0.598 & 0.648 & 0.691 & 0.708 & 0.709 \\

\midrule
\multicolumn{9}{l}{\textbf{$N=16$ weak samples per prompt} \quad ($b=0$: 0.448)} \\
Corrected disagreement & ours
    & \textbf{0.468} & \textbf{0.486} & \textbf{0.503} & \textbf{0.519} & \textbf{0.536} & \textbf{0.551} & \textbf{0.608} \\
Weak--strong disagreement & ours
    & 0.464 & 0.484 & 0.502 & 0.518 & 0.534 & 0.550 & 0.606 \\
Self-consistency & baseline
    & 0.459 & 0.471 & 0.483 & 0.495 & 0.507 & 0.518 & 0.565 \\
Semantic entropy & baseline
    & 0.458 & 0.470 & 0.481 & 0.493 & 0.504 & 0.516 & 0.563 \\
Token entropy & baseline
    & 0.463 & 0.477 & 0.491 & 0.506 & 0.521 & 0.535 & 0.592 \\
Proxy reward$^\dagger$ & baseline
    & 0.460 & 0.473 & 0.488 & 0.504 & 0.519 & 0.535 & 0.594 \\
Oracle routing gain & oracle
    & 0.499 & 0.549 & 0.599 & 0.648 & 0.691 & 0.708 & 0.709 \\

\bottomrule
\end{tabular}
\vspace{0.5em}
\caption*{\footnotesize Consistency-based scores are undefined for $N=1$ and are therefore omitted in that panel.
For $N>1$, judge-based clusters do not improve the kept-program accuracy, so gains come from improved routing signals rather than a better plurality answer.
Bold numbers mark the best non-oracle routing score at each budget within a fixed $N$.
$^\dagger$Proxy reward denotes the reward-model-based routing baseline.}
\end{table*}

\subsection{Ablations on Selection Policy}
\label{app:ablations-on-selection-policy}

In the main experiments, we use the same selection policy across
routing methods to isolate the effect of the routing score.
Here, we examine how the choice of weak response affects cascade
accuracy.

Recall from Section~\ref{sec:routing_fundamentals} that, given
$h_\sigma=(p,y_{w,1},\ldots,y_{w,N})$, the optimal selection policy is
\[
\phi^*(h_\sigma)\in
\arg\max_{i\leq N}\Pr(Y=y_{w,i}\mid P=p).
\]
Thus, the selection policy should choose the weak response most
likely to be correct. Learning this probability directly requires
correctness labels. We compare several selection rules that do
not require ground-truth supervision, together with a supervised
correctness probe:
\begin{itemize}

    \item \textbf{Majority vote.}
    Group equivalent responses and select a response from the
    largest group.

    \item \textbf{Token entropy.}
    Select the response with the lowest mean token entropy
    under the weak model.

    \item \textbf{Disagreement (ours).}
    Select the response with the lowest predicted weak--strong
    disagreement,
    $\phi(h_\sigma)\in
    \arg\min_{i\leq N}
    \widehat{f}_d(\mathbf{h}(p,y_{w,i}))$.
    This favors responses that the strong model is more likely
    to agree with.

    \item \textbf{Correctness probe.}
    Select the response with the highest predicted probability
    of correctness. This probe uses ground-truth correctness
    labels during training and serves as a supervised comparison.
\end{itemize}

\paragraph{Selection across routing scores.}
Figure~\ref{fig:math500-selection-n8} compares these selection
policies on \textsc{Math}-500 with $N=8$ weak responses per prompt.
Each panel fixes the routing score and varies the selection policy
and routing budget. Majority vote achieves the highest accuracy
at every reported budget for all four routing scores. This is
consistent with the effectiveness of grouping equivalent final
answers in mathematical reasoning. Disagreement-based selection
performs similarly to the supervised correctness probe, and both
improve on random selection and token entropy.

\paragraph{Selection across sample counts.}
Figures~\ref{fig:spider-selection-alln}
and~\ref{fig:acecode-selection-alln} extend the comparison to
\textsc{Spider} and \textsc{AceCode}. We fix the routing score to
corrected disagreement and vary both the routing budget and the
number of weak responses $N$. These experiments examine whether
selection policies can benefit from additional candidates in
tasks where identifying equivalent responses is more difficult.

\begin{figure*}[t]
\centering
\small
\setlength{\tabcolsep}{3.5pt}
\renewcommand{\arraystretch}{1.08}

\begin{minipage}[t]{0.48\textwidth}
\centering
\textbf{Routing: corrected disagreement (ours)}
\vspace{0.3em}

\includegraphics[width=\linewidth]{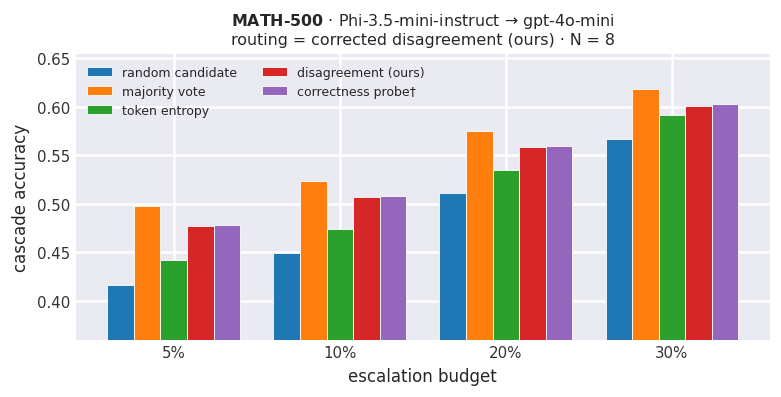}

\vspace{0.3em}
\resizebox{\linewidth}{!}{
\begin{tabular}{lccccc}
\toprule
Selection & $b=0\%$ & $b=5\%$ & $b=10\%$ & $b=20\%$ & $b=30\%$ \\
\midrule
Random candidate            & 0.377 & 0.416 & 0.449 & 0.511 & 0.567 \\
Majority vote               & \textbf{0.467} & \textbf{0.498} & \textbf{0.524} & \textbf{0.575} & \textbf{0.619} \\
Token entropy               & 0.404 & 0.443 & 0.474 & 0.535 & 0.592 \\
Disagreement (ours)         & 0.449 & 0.478 & 0.507 & 0.558 & 0.601 \\
Correctness probe$^\dagger$ & 0.450 & 0.478 & 0.508 & 0.560 & 0.602 \\
\bottomrule
\end{tabular}
}
\end{minipage}
\hfill
\begin{minipage}[t]{0.48\textwidth}
\centering
\textbf{Routing: weak--strong disagreement (ours)}
\vspace{0.3em}

\includegraphics[width=\linewidth]{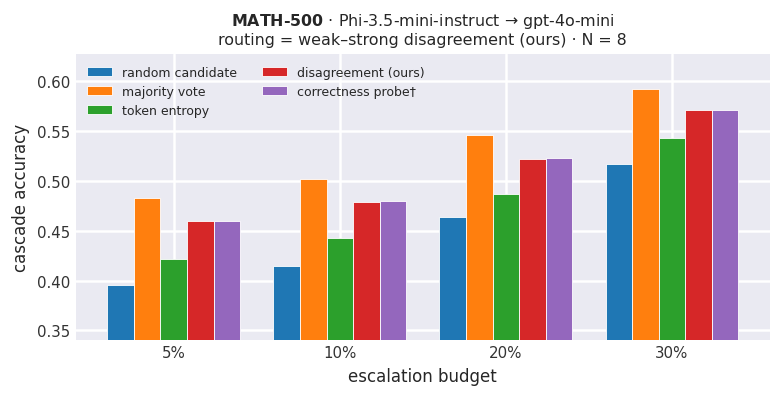}

\vspace{0.3em}
\resizebox{\linewidth}{!}{
\begin{tabular}{lccccc}
\toprule
Selection & $b=0\%$ & $b=5\%$ & $b=10\%$ & $b=20\%$ & $b=30\%$ \\
\midrule
Random candidate            & 0.377 & 0.395 & 0.415 & 0.464 & 0.517 \\
Majority vote               & \textbf{0.467} & \textbf{0.483} & \textbf{0.502} & \textbf{0.546} & \textbf{0.593} \\
Token entropy               & 0.404 & 0.422 & 0.443 & 0.487 & 0.543 \\
Disagreement (ours)         & 0.449 & 0.460 & 0.479 & 0.522 & 0.571 \\
Correctness probe$^\dagger$ & 0.450 & 0.460 & 0.480 & 0.523 & 0.571 \\
\bottomrule
\end{tabular}
}
\end{minipage}

\vspace{1.1em}

\begin{minipage}[t]{0.48\textwidth}
\centering
\textbf{Routing: proxy reward$^\dagger$}
\vspace{0.3em}

\includegraphics[width=\linewidth]{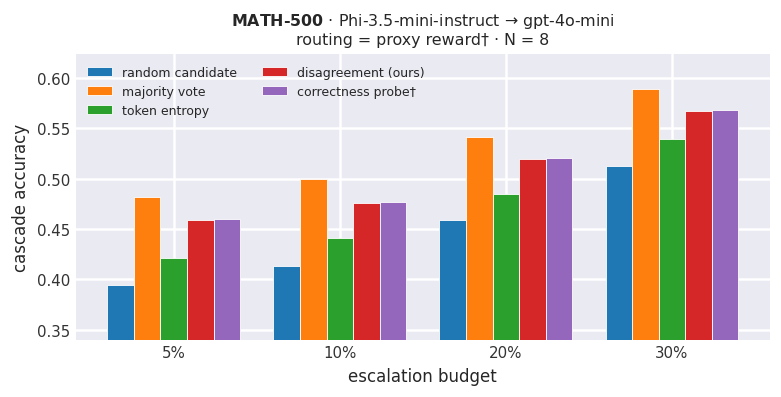}

\vspace{0.3em}
\resizebox{\linewidth}{!}{
\begin{tabular}{lccccc}
\toprule
Selection & $b=0\%$ & $b=5\%$ & $b=10\%$ & $b=20\%$ & $b=30\%$ \\
\midrule
Random candidate            & 0.377 & 0.394 & 0.413 & 0.459 & 0.513 \\
Majority vote               & \textbf{0.467} & \textbf{0.482} & \textbf{0.500} & \textbf{0.542} & \textbf{0.589} \\
Token entropy               & 0.404 & 0.421 & 0.441 & 0.485 & 0.539 \\
Disagreement (ours)         & 0.449 & 0.459 & 0.476 & 0.520 & 0.567 \\
Correctness probe$^\dagger$ & 0.450 & 0.460 & 0.477 & 0.521 & 0.568 \\
\bottomrule
\end{tabular}
}
\end{minipage}
\hfill
\begin{minipage}[t]{0.48\textwidth}
\centering
\textbf{Routing: self-consistency}
\vspace{0.3em}

\includegraphics[width=\linewidth]{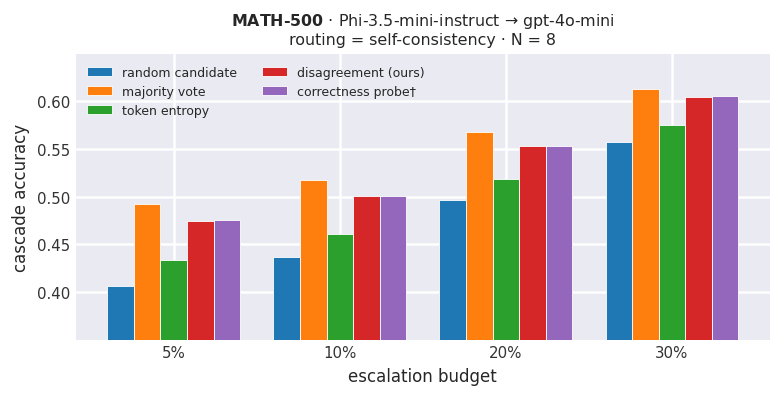}

\vspace{0.3em}
\resizebox{\linewidth}{!}{
\begin{tabular}{lccccc}
\toprule
Selection & $b=0\%$ & $b=5\%$ & $b=10\%$ & $b=20\%$ & $b=30\%$ \\
\midrule
Random candidate            & 0.377 & 0.407 & 0.437 & 0.496 & 0.557 \\
Majority vote               & \textbf{0.467} & \textbf{0.492} & \textbf{0.518} & \textbf{0.567} & \textbf{0.613} \\
Token entropy               & 0.404 & 0.433 & 0.461 & 0.519 & 0.575 \\
Disagreement (ours)         & 0.449 & 0.475 & 0.501 & 0.553 & 0.604 \\
Correctness probe$^\dagger$ & 0.450 & 0.475 & 0.501 & 0.553 & 0.605 \\
\bottomrule
\end{tabular}
}
\end{minipage}

\caption{
Effect of weak-response selection on \textsc{Math}-500 for
\textbf{Phi-3.5-mini-instruct} $\rightarrow$ \textbf{GPT-4o-mini},
using $N=8$ weak responses per prompt and 500 test prompts.
Each panel fixes the routing score and reports cascade accuracy
for different selection policies and routing budgets.
At $b=0$, accuracy depends only on the selection policy.
Bold values indicate the highest accuracy in each column.
$^\dagger$ denotes methods that use ground-truth supervision
during training.
}
\label{fig:math500-selection-n8}
\end{figure*}
\begin{figure}[t]
    \centering
    \includegraphics[width=\linewidth]{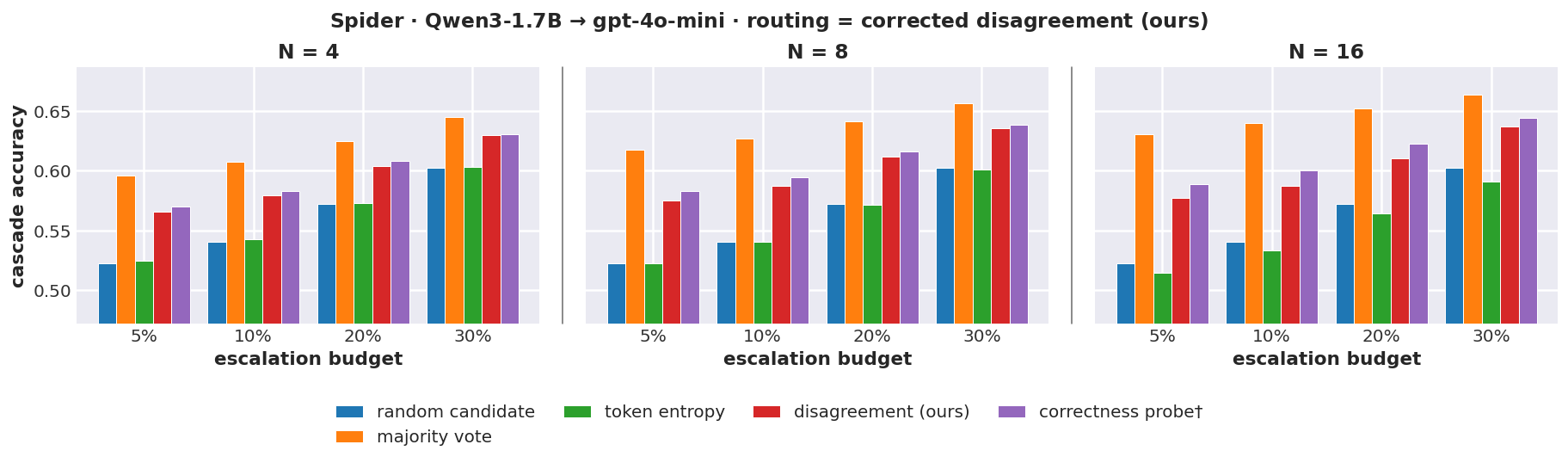}
    \caption{
    Effect of weak-response selection on \textsc{Spider}.
    Routing uses corrected disagreement throughout.
    We compare selection policies across routing budgets and
    different numbers of weak-model samples $N$.
    }
    \label{fig:spider-selection-alln}
\end{figure}

\begin{figure}[t]
    \centering
    \includegraphics[width=\linewidth]{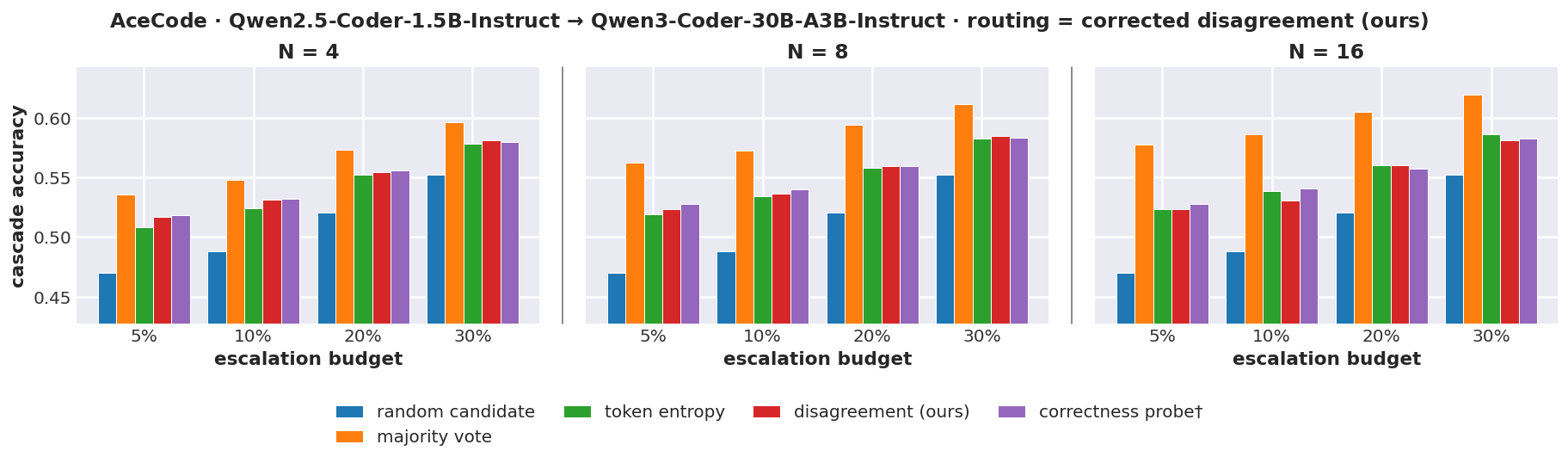}
    \caption{
    Effect of weak-response selection on \textsc{AceCode}.
    Routing uses corrected disagreement throughout.
    We compare selection policies across routing budgets and
    different numbers of weak-model samples $N$.
    }
    \label{fig:acecode-selection-alln}
\end{figure}
\subsection{Additional Results for the Agentic Experiments}
\label{app:agentic_additional}

In this section, we provide complementary results for the multi-step agentic experiments presented in Section~\ref{subsec:agentic_experiments}. 
Table~\ref{tab:app-webshop-additional} reports additional results,
including the token-entropy baseline omitted from the main text.
Alongside full-trajectory success, we report the mean task score for completeness. 

\begin{table}[h!]
\centering
\small
\caption{WebShop routing performance (200 evaluation tasks) comparing realized routing budget, end-to-end success rate, and average task score across multiple label-free baselines.}
\label{tab:app-webshop-additional}
\begin{tabular}{lcccccc}
\toprule
\textbf{Routing Strategy} & \multicolumn{3}{c}{\textbf{Success Rate (SE)}} & \multicolumn{3}{c}{\textbf{Average Task Score}} \\
\midrule
\textit{Realized Budget ($\approx$)} & \textbf{20\%} & \textbf{40\%} & \textbf{60\%} & \textbf{20\%} & \textbf{40\%} & \textbf{60\%} \\
\midrule
Corrected disagreement (ours) & \textbf{0.210} (0.029) & \textbf{0.460} (0.035) & \textbf{0.765} (0.030) & \textbf{0.665} & \textbf{0.748} & \textbf{0.828} \\
Weak--strong disagreement (ours) & 0.065 (0.017) & 0.260 (0.031) & 0.715 (0.032) & 0.466 & 0.716 & 0.880 \\
Self-consistency & 0.045 (0.015) & 0.145 (0.025) & 0.410 (0.035) & 0.448 & 0.536 & 0.755 \\
Token entropy & 0.040 (0.014) & 0.150 (0.025) & 0.405 (0.035) & 0.448 & 0.535 & 0.698 \\
Random baseline & 0.160 (0.026) & 0.200 (0.028) & 0.505 (0.035) & 0.543 & 0.590 & 0.776 \\
\bottomrule
\end{tabular}
\end{table}

\subsection{Probe Representations and Training}
\label{app:probes}

We provide additional details on the probes used to estimate the
routing signals in Section~\ref{sec:finite-sample-algorithms}. The probes are trained on empirical
targets computed from offline model generations, while the weak model remains frozen. We extract both probe inputs from the weak model's hidden states.
For the strong-model self-disagreement probe, we use the hidden
state at the final prompt token, denoted by $\mathbf{h}(p)$.
For the weak--strong disagreement probe, we use the mean hidden
state over the generated response tokens, denoted by
$\mathbf{h}(p,y_w)$. The second representation depends on both
the prompt and the generated response from the weak model. In the agentic environments,
the trajectory history $H_t$ serves as the prompt and the proposed
action serves as the response.

We train each probe:
\begin{equation}
\widehat{f}
\in \arg\min_{f\in\mathcal{F}}
\frac{1}{|\mathcal{D}|}
\sum_{(z,\widehat{u})\in\mathcal{D}}
\ell\bigl(f(\mathbf{h}(z)),\widehat{u}\bigr)
\qquad
\label{eq:appendix-probe-loss}
\end{equation}
by minimizing binary cross-entropy with empirical
probabilities as soft targets.
For $\widehat{f}_d$, each example pairs a prompt--response
representation $\mathbf{h}(p,y_w)$ with the target
$\widehat{d}(p,y_w)$ defined in
Equation~\ref{eq:disagreement-target}.
For $\widehat{f}_g$, each example pairs a prompt representation
$\mathbf{h}(p)$ with the target $\widehat{g}(p)$ defined in
Equation~\ref{eq:self-disagreement-target}.
Probe outputs are restricted to $[0,1]$.
For brevity, the main text writes these predictors as
$\widehat{f}_d(p,y_w)$ and $\widehat{f}_g(p)$, leaving the
representation extraction implicit.

Additionally, when ground-truth answers are available, we also train the
proxy-reward probe $\widehat{f}_c$ using the same prompt--response
representation and binary cross-entropy loss, with correctness
target $1-A(y_w,Y)$. This probe provides a supervised comparison;
the two disagreement probes require no ground-truth correctness
labels. At inference, computing the learned routing scores requires only
weak-model representations; the strong model is queried only
when the router chooses to escalate. We will release the
training code and experiment configurations upon acceptance,
which will including feature-extraction layers used for each model and hyperparameter details. 

\subsection{On-Policy Data Collection for Agentic Probes}
\label{app:dagger}
To ensure our learned probes generalize to test-time deployments where strong-model interventions alter the trajectory state, we collect training data using an iterative, on-policy schedule inspired by DAgger \citep{ross2011reductionimitationlearningstructured, li2026revisitingdaggererallmagents}.

During test-time routing, executing strong-model actions steers the agent into states that the weak model would not visit autonomously. Fitting probes exclusively on purely weak-model trajectories results in a distribution mismatch (the classic compounding-shift problem of behavioral cloning \citep{ross2011reductionimitationlearningstructured}.) We mitigate this by collecting training data across iterative phases: each phase rolls out the current routed policy, labels every visited state, and trains subsequent probes on the aggregated dataset from all preceding phases.

\paragraph{Dataset aggregation.}
In multi-step environments, let $H_t = (p, y^1, o^1, \dots, y^{t-1}, o^{t-1})$ denote the trajectory history at step $t$, comprising the instruction prompt $p$, past actions $y_i$, and observations $o_i$. At each step, both models propose actions: the weak model generates $K$ sampled actions $\{y_1^w, \dots, y_K^w\}$ and $K$ sampled actions $\{y_1^s, \dots, y_K^s\}$ from the strong model. ($K=10$ across all environments). For simplicity and consistency across agentic baselines, we fix the weak-model candidate selection policy $\phi(H_t)$ to execute the first sampled action ($y_1^w$). Using the task-specific transition agreement indicator $A$ defined in Section~\ref{sec:finite-sample-algorithms}, we compute empirical disagreement targets $\widehat{d}(H_t, y_i^w)$ and $\widehat{g}(H_t)$. We record the weak model's hidden states: the final history token representation $\mathbf{h}(H_t)$ for the self-disagreement probe $\widehat{f}_g$, and the mean hidden-state representation over action tokens $\mathbf{h}(H_t, y_i^w)$ for the weak--strong disagreement probe $\widehat{f}_d$.

\paragraph{Stochastic routing during collection.}
In the first phase, the weak model acts at every step. In subsequent
phases, we fit the probes on the accumulated training data and keep
them fixed throughout collection. At each visited history $H_t$,
we map the corrected disagreement score to a routing probability:
\begin{equation}
\label{eq:collection-routing}
p_{\mathrm{route}}^{(m)}(H_t,y_1^w)
=
\begin{cases}
0, & m=1, \\[2mm]
\displaystyle
\frac{
1+\widehat{f}_d^{(m)}\!\left(\mathbf{h}(H_t,y_1^w)\right)
-\widehat{f}_g^{(m)}\!\left(\mathbf{h}(H_t)\right)
}{2},
& m\geq 2.
\end{cases}
\end{equation}
Both probes take values in $[0,1]$, so the corrected disagreement
score lies in $[-1,1]$ and the routing probability lies in $[0,1]$.
We sample a routing decision from this probability and execute
either the strong model's greedy action or the first weak sample.

\begin{algorithm}[ht!]
\caption{On-Policy Collection of Disagreement Data}
\label{alg:dagger-collection}
\begin{algorithmic}[1]
\REQUIRE Tasks
$\mathcal{T}_{\mathrm{train}},\mathcal{T}_{\mathrm{test}}$;
weak and strong models; selection policy $\phi$;
samples $K$; phases $M$
\STATE $\mathcal{D}\gets\emptyset$
\FOR{$m=1,\ldots,M$}
    \IF{$m\geq 2$}
        \STATE Fit $\widehat{f}_d^{(m)},\widehat{f}_g^{(m)}$
        on training-task records in $\mathcal{D}$
    \ENDIF
    \FOR{each task $u\in
        \mathcal{T}_{\mathrm{train}}\cup\mathcal{T}_{\mathrm{test}}$}
        \STATE Initialize trajectory history $H_1$; set $t\gets 1$
        \WHILE{task $u$ is not finished}
            \STATE Draw $y_{w,1},\ldots,y_{w,K}$
            independently from the weak model given $H_t$
            \STATE Draw $y_{s,1},\ldots,y_{s,K}$
            independently from the strong model given $H_t$
            \STATE $h_{t,K}\gets(H_t,y_{w,1},\ldots,y_{w,K})$;
            $i_t\gets\phi(h_{t,K})$
            \STATE Compute $\widehat{g}(H_t)$ and
            $\{\widehat{d}(H_t,y_{w,i})\}_{i=1}^{K}$
            from the sampled actions
            \STATE Add these targets and their weak-model
            representations to $\mathcal{D}$, tagged by task $u$
            \STATE $p_t\gets 0$
            \IF{$m\geq 2$}
                \STATE $\displaystyle
                p_t\gets\frac{
                    1+
                    \widehat{f}_d^{(m)}
                    \bigl(\mathbf{h}(H_t,y_{w,i_t})\bigr)
                    -
                    \widehat{f}_g^{(m)}
                    \bigl(\mathbf{h}(H_t)\bigr)
                }{2}$
            \ENDIF
            \STATE Draw $b_t\sim\operatorname{Bernoulli}(p_t)$
            \IF{$b_t=1$}
                \STATE Draw $y_s$ from the strong model given $H_t$
                \STATE $y^t\gets y_s$
            \ELSE
                \STATE $y^t\gets y_{w,i_t}$
            \ENDIF
            \STATE Execute $y^t$ and observe $o^t$
            \STATE $H_{t+1}\gets(H_t,y^t,o^t)$;
            $t\gets t+1$
        \ENDWHILE
    \ENDFOR
\ENDFOR
\STATE \textbf{return} $\mathcal{D}$
\end{algorithmic}
\end{algorithm}

\newpage
\section{Proof of Theorem~\ref{thm:optimal_routing}}
\label{app:optimal_routing}

We first give a more detailed statement of the result, including the
optimal sampling policy. The proof has a simple structure. For a fixed
number of weak samples, we condition on the sampled responses: the best
weak response is the one with the largest correctness probability, and
routing is optimal when the improvement offered by the strong model is
large enough. To optimize sampling, we introduce prices for weak-model
samples and strong-model queries. For each prompt, the optimal number of
weak samples then maximizes its expected value net of the sampling price.

For the sampling characterization, assume that at most $K<\infty$ weak
responses may be generated. Conditional on $P=p$, the weak responses are
i.i.d.\ and independent of $(Y,Y_s)$. Recall
\[
q_p(y):=\Pr(Y=y\mid P=p),
\qquad
t_p:=\Pr(Y_s=Y\mid P=p).
\]

\begin{theorem}[Full version of Theorem~\ref{thm:optimal_routing}]
\label{thm:optimal_routing_full}
Let $1\leq B_w\leq K$ and $0\leq B_s\leq1$.

\emph{(i) Fixed sampling.}
Fix any sampling policy $\sigma$ satisfying
$\mathbb E[\sigma(P)]\leq B_w$. For a realized history
\[
h_\sigma=(p,y_{w,1},\ldots,y_{w,\sigma(p)}),
\]
define
\begin{equation}
\label{eq:app_routing_gain}
\Delta(h_\sigma)
:=
t_p-\max_{i\leq\sigma(p)}q_p(y_{w,i}).
\end{equation}
Then an optimal selection and routing policy is
\begin{equation}
\label{eq:app_optimal_selection_routing}
\phi^*(h_\sigma)
\in
\operatorname*{arg\,max}_{i\leq\sigma(p)}
q_p(y_{w,i}),
\qquad
\pi^*(h_\sigma)
=
\begin{cases}
1, & \Delta(h_\sigma)>\tau,\\
0, & \Delta(h_\sigma)<\tau,\\
\rho(h_\sigma), & \Delta(h_\sigma)=\tau,
\end{cases}
\end{equation}
where $\tau\geq0$ and $\rho(h_\sigma)\in[0,1]$ can be chosen so that
\begin{equation}
\label{eq:app_routing_budget}
\mathbb E[\pi^*(H_\sigma)]\leq B_s,
\qquad
\tau\bigl(\mathbb E[\pi^*(H_\sigma)]-B_s\bigr)=0.
\end{equation}

\emph{(ii) Optimal sampling.}
For $n\in\{1,\ldots,K\}$, define
\[
M_{p,n}:=\max_{i\leq n}q_p(Y_{w,i})
\]
and
\begin{equation}
\label{eq:app_sampling_value}
U_p(n;\tau)
:=
\mathbb E\!\left[
\max\{M_{p,n},\,t_p-\tau\}
\,\middle|\,P=p
\right].
\end{equation}
Then, for some global weak-sampling price $\beta\geq0$, an optimal
sampling policy satisfies
\begin{equation}
\label{eq:app_optimal_sampling}
\sigma^*(p)
\in
\operatorname*{arg\,max}_{1\leq n\leq K}
\left\{
U_p(n;\tau)-\beta n
\right\}.
\end{equation}
Together with the selection and routing rules in
\eqref{eq:app_optimal_selection_routing}, the prices can be chosen so that
\begin{equation}
\label{eq:app_complementarity}
\mathbb E[\sigma^*(P)]\leq B_w,
\qquad
\beta\bigl(\mathbb E[\sigma^*(P)]-B_w\bigr)=0,
\end{equation}
and~\eqref{eq:app_routing_budget} holds.

If needed at a budget boundary, the policy may randomize between tied
maximizers in~\eqref{eq:app_optimal_sampling} or at
$\Delta(h_\sigma)=\tau$.
\end{theorem}

\begin{proof}
We prove the result in a few simple steps.

\emph{Claim 1. For fixed sampling, the selection rule in
\eqref{eq:app_optimal_selection_routing} is optimal.}

Condition on a realized history
$h_\sigma=(p,y_{w,1},\ldots,y_{w,\sigma(p)})$. Because the weak-model
samples are independent of $(Y,Y_s)$ conditional on $P$, the conditional
correctness probabilities are
\[
\Pr(Y_{w,i}=Y\mid H_\sigma=h_\sigma)
=
q_p(y_{w,i}),
\qquad
\Pr(Y_s=Y\mid H_\sigma=h_\sigma)
=
t_p.
\]
Hence, for any selection and routing decisions, the conditional expected
accuracy is
\[
(1-\pi(h_\sigma))
q_p(y_{w,\phi(h_\sigma)})
+
\pi(h_\sigma)t_p.
\]
For any fixed routing decision, this is maximized by selecting a weak
response with the largest value of $q_p(y_{w,i})$. Therefore
\[
\phi^*(h_\sigma)
\in
\operatorname*{arg\,max}_{i\leq\sigma(p)}
q_p(y_{w,i}).
\]

With this choice, the conditional expected accuracy becomes
\[
\max_{i\leq\sigma(p)}q_p(y_{w,i})
+
\pi(h_\sigma)\Delta(h_\sigma).
\]
Thus, after optimizing selection, routing only needs to decide where to
spend the strong-model budget based on the gain $\Delta(h_\sigma)$.

\emph{Claim 2. The optimal routing policy is a threshold rule on
$\Delta(h_\sigma)$.}

For fixed $\sigma$, maximizing accuracy is equivalent to
\[
\max_{\pi}
\mathbb E[
\pi(H_\sigma)\Delta(H_\sigma)
]
\qquad
\text{subject to}
\qquad
\mathbb E[\pi(H_\sigma)]\leq B_s.
\]
Therefore the budget should be allocated to histories with the largest
values of $\Delta$.

More formally, choose $\tau\geq0$ and, if necessary, randomize on
$\{\Delta=\tau\}$ so that~\eqref{eq:app_routing_budget} holds. For any
other feasible routing policy $\pi$,
\[
(\pi^*-\pi)(\Delta-\tau)\geq0
\]
pointwise. Hence
\begin{align*}
\mathbb E[(\pi^*-\pi)\Delta]
&=
\mathbb E[(\pi^*-\pi)(\Delta-\tau)]
+
\tau\bigl(
\mathbb E[\pi^*]-\mathbb E[\pi]
\bigr)
\\
&\geq0,
\end{align*}
where the second term is nonnegative by feasibility and complementary
slackness. Thus $\pi^*$ is optimal, proving part~(i).

\emph{Claim 3. For fixed prices $(\beta,\tau)$, the optimal number of
weak samples satisfies~\eqref{eq:app_optimal_sampling}.}

Introduce a price $\beta$ for each weak-model sample and a price $\tau$
for each strong-model query. Consider a prompt $p$ for which we choose
$n$ weak samples.

After the samples are observed, Claim~1 shows that the best weak-model
accuracy is
\[
M_{p,n}
=
\max_{i\leq n}q_p(Y_{w,i}).
\]
If we return the weak response, the priced value is $M_{p,n}$. If we
query the strong model, the priced value is $t_p-\tau$. Therefore the
optimal terminal value is
\[
\max\{M_{p,n},t_p-\tau\}.
\]
Before the weak responses are sampled, its expected value is exactly
$U_p(n;\tau)$ from~\eqref{eq:app_sampling_value}. After subtracting the
sampling cost $\beta n$, the value of choosing $n$ samples is
\[
U_p(n;\tau)-\beta n.
\]
Since $\sigma$ depends only on the prompt, the optimal choice can be made
separately for each $p$, giving
\[
\sigma^*(p)
\in
\operatorname*{arg\,max}_{1\leq n\leq K}
\{U_p(n;\tau)-\beta n\}.
\]

The sampling rule also has a useful marginal interpretation. Define
\[
G_p(n;\tau)
:=
U_p(n+1;\tau)-U_p(n;\tau),
\qquad n<K.
\]
These marginal gains are nonincreasing in $n$. To see this, write
\[
Z_i:=q_p(Y_{w,i}),
\qquad
R_n:=\max\{t_p-\tau,Z_1,\ldots,Z_n\}.
\]
Then
\[
G_p(n;\tau)
=
\mathbb E[(Z_{n+1}-R_n)_+\mid P=p].
\]
Since $R_{n+1}\geq R_n$ and a fresh weak response has the same
distribution at every step, the expected improvement from one additional
sample cannot increase with $n$. Thus~\eqref{eq:app_optimal_sampling} can
equivalently be viewed as increasing the planned number of weak samples
while the expected marginal gain exceeds $\beta$, and stopping once it
falls below $\beta$.

\emph{Claim 4. Suitable prices recover the constrained optimum.}

Let
\[
A(\sigma,\phi,\pi)
\]
denote expected accuracy, and define
\[
W(\sigma):=\mathbb E[\sigma(P)],
\qquad
S(\sigma,\pi):=\mathbb E[\pi(H_\sigma)].
\]
For prices $\beta,\tau\geq0$, consider the priced objective
\[
A(\sigma,\phi,\pi)
-
\beta W(\sigma)
-
\tau S(\sigma,\pi).
\]
Claims~1--3 characterize a policy that maximizes this objective.

Under the standard randomized-policy formulation with finite $K$, the
set of achievable triples $(W,S,A)$ is convex and compact. Therefore a
supporting-hyperplane argument gives prices $\beta,\tau\geq0$ and
boundary randomization, when necessary, satisfying
\eqref{eq:app_routing_budget} and~\eqref{eq:app_complementarity}.

It remains only to verify that these conditions imply constrained
optimality. Let $(\sigma^*,\phi^*,\pi^*)$ be the resulting policy, and
let $(\sigma,\phi,\pi)$ be any other feasible policy. Since the former
maximizes the priced objective,
\begin{align*}
A(\sigma,\phi,\pi)
&\leq
A(\sigma^*,\phi^*,\pi^*)
-\beta W(\sigma^*)
-\tau S(\sigma^*,\pi^*)
\\
&\qquad
+\beta W(\sigma)
+\tau S(\sigma,\pi)
\\
&\leq
A(\sigma^*,\phi^*,\pi^*)
+\beta(B_w-W(\sigma^*))
+\tau(B_s-S(\sigma^*,\pi^*))
\\
&=
A(\sigma^*,\phi^*,\pi^*),
\end{align*}
where the last equality follows from complementary slackness. Hence the
stated sampling, selection, and routing policies are jointly optimal.
\end{proof}
\end{document}